\documentclass{article}

\usepackage[preprint]{neurips_2026}
\usepackage{mathtools}

\usepackage[utf8]{inputenc} %
\usepackage[T1]{fontenc}    %
\usepackage{hyperref}       %
\usepackage{url}            %
\usepackage{booktabs}       %
\usepackage{subcaption}     %
\usepackage{amsfonts}       %
\usepackage{nicefrac}       %
\usepackage{microtype}      %
\usepackage{xcolor}         %

\usepackage{amsmath}
\usepackage{amssymb}
\usepackage{amsthm}
\usepackage{geometry}
\usepackage{xcolor}
\usepackage{xspace}
\usepackage{amssymb}
\usepackage{enumitem}

\renewcommand{\tilde}{\widetilde}

\newtheorem{theorem}{Theorem}
\newtheorem{lemma}{Lemma}

\theoremstyle{definition}

\newcommand{\R}{\mathbb{R}}
\newcommand{\E}{\mathbb{E}}
\newcommand{\Pbb}{\mathbb{P}}
\newcommand{\vnmse}{\operatorname{vNMSE}}

\newcommand{\diag}{\operatorname{diag}}
\newcommand{\Var}{\operatorname{Var}}
\newcommand{\Cov}{\operatorname{Cov}}

\newcommand{\set}[1]{\left\{#1\right\}}
\newcommand{\norm}[1]{\left\lVert #1\right\rVert}

\newcommand{\NH}{\tilde{H}\xspace}
\usepackage{cleveref}

\title{Quantizing With Randomized Hadamard Transforms: Efficient Heuristic Now Proven}

\author{%
  Ran Ben-Basat\\
  University College London\\
  \And
  William Kuszmaul\\
  Carnegie Mellon University\\
  \And
  Michael Mitzenmacher\\
  Harvard University\\
  \And
  Amit Portnoy\\
  Independent Researcher\\
  \And
  Shay Vargaftik\\
  VMware Research by Broadcom\\
}

\makeatletter
\def\@noticestring{}
\makeatother

\begin{document}

\maketitle

\begin{abstract}
Uniform random rotations (URRs) are a common preprocessing step in modern quantization approaches used for gradient compression, inference acceleration, KV-cache compression, model weight quantization, and approximate nearest-neighbor search in vector databases. In practice, URRs are often replaced by randomized Hadamard transforms (RHTs), which preserve orthogonality while admitting fast implementations. The remaining issue is the performance for worst-case inputs.  
With a URR, each coordinate is individually distributed as a shifted beta distribution, which converges to a Gaussian distribution in high dimensions.  Generally, one RHT is not suitable in the worst case, as individual coordinates can be far from these distributions.  
We show that after composing two RHTs on any $d$-sized input vector, the marginal distribution of every fixed coordinate of the normalized rotated vector is within $\mathcal O(d^{-1/2})$ of a standard Gaussian both in Kolmogorov distance and in $1$-Wasserstein distance. We then plug these bounds into the analyses of modern compression schemes, namely DRIVE and QUIC-FL, and show that two RHTs achieve performance that asymptotically matches URRs. 

However, while two RHTs suffice for scalar quantization, they may be insufficient for Vector Quantization (VQ), which depends on the joint distribution across fixed-size blocks of coordinates (as opposed to only marginal distribution convergence for single coordinates). We prove that a composition of three RHTs leads to decaying coordinate covariance. This ensures that any fixed, bounded, multi-dimensional VQ codebook optimized for URRs has the same expected error when using three RHTs, up to an additive term that asymptotically vanishes with the dimension.

Finally, because practical inputs are rarely adversarial, we propose a linear-time $\mathcal{O}(d)$ check on the input's moments to dynamically adapt the number of RHTs used at runtime to improve performance.

\end{abstract}

\section{Introduction} 

A foundational technique in quantization is applying a uniform random rotation (URR) to the input vectors beforehand, which ensures the vector's squared $\ell_2$ norm is evenly distributed across coordinates in expectation. Specifically, for a URR $\mathcal{R}_U$ and input $x \in \mathbb{R}^d$, $\E[(\mathcal{R}_Ux)_i^2]=\norm{x}^2/d$.
Moreover, because individual coordinates of a rotated vector rapidly converge to a Gaussian distribution as the vector's dimension grows, rotations enable stable, nearly dimension-free scalar quantization designs. This robust statistical smoothing has recently found numerous applications, including Distributed Mean Estimation (DME) and gradient compression in federated learning (e.g., QUIC-FL~\cite{ben2024accelerating}), 
distributed training (e.g.,~\cite{han2024beyond,han2026dynamiq}),  privacy (e.g., ScionFL~\cite{ben2024scionfl}),
parameter quantization (e.g., Quartet-2 via MS-EDEN~\cite{quartet2026}) and activation quantization (e.g., HACK~\cite{zhang2025hack}),
 database vector search (e.g., RaBitQ~\cite{gao2024rabitq}), KV cache quantization 
(e.g., TurboQuant~\cite{zandieh2025turboquant} and EDEN~\cite{ben2026note}), 
and advanced post-training quantization frameworks~(e.g.,~HIGGS~\cite{malinovskii2025higgs}).

While URRs are conceptually and mathematically attractive, they take $O(d^2)$ operations to compute for $d$-dimensional vectors and thus are often too slow. This is why the literature repeatedly suggests replacing them with randomized Hadamard transforms (RHTs), which are structured random rotations that can be computed more efficiently, i.e., in $O(d \log d)$ operations. RHTs have been suggested for many quantization algorithms, e.g., 
\cite{ashkboos2024quarot,basat2026torque,basat2024optimal,ben2025better,ben2024scionfl,caldas2018expanding,dorfman2023docofl, gao2025practical,konevcny2016federated, ordentlich2025optimal, tseng2024quip, tseng2024qtip,vargaftik2022eden, yang2025raana, zandieh2025turboquant}.
Furthermore, practitioners already use successive RHTs heuristically. For example, FAISS implements three RHTs as a fast rotation~\cite{faiss2026hadamard}. FALCONN's original cross-polytope LSH construction uses three RHTs~\cite{andoni2015practical}, while Weaviate uses a related construction with blockwise Hadamard transforms and random permutations~\cite{christiani2025rq}.
More broadly, RHTs also underlie fast Johnson-Lindenstrauss transforms~\cite{ailon2006approximate,ailon2011almost} and subsampled randomized Hadamard transforms (SRHTs)~\cite{tropp2011improved}. The algorithmic advantage is immediate: the transform is orthogonal, \mbox{fast, memory efficient, and GPU-friendly.}

At a theoretical level, however, the delicate issue is that RHTs are more difficult to work with, as they do not fully provide the same properties as a URR.  
A single RHT may not be suitable for worst-case inputs;  for example, the unit vector $(1/\sqrt{2},1/\sqrt{2},0,0,\ldots,0)$ after an RHT will have all coordinates take on values in the set $\{0,\pm\sqrt{2/d}\}$, rather than a (near-)Gaussian coordinate law. Intuitively, even with a smoother initial vector, there can be a gap in terms of readily obtaining results for applications that utilize the coordinate marginal distribution of a random rotation.     

In this paper, we show that, for scalar distributions, composing two successive RHTs eliminates such bad cases. Specifically, after one RHT, the resulting intermediate vector has a small expected third moment. Given this intermediate vector, each coordinate of the final vector becomes an ordinary Rademacher sum. We then apply the appropriate Berry--Esseen inequalities, yielding the necessary bounds for comparing these coordinate distributions to a Gaussian. We analyze this comparison using two metrics: the Kolmogorov distance, suited to thresholded and bounded-variation functionals, and the $1$-Wasserstein distance, suited to Lipschitz functionals. With both estimates established, we then apply them to obtain new formal guarantees for recent scalar quantization schemes (which quantize each entry separately), namely DRIVE~\cite{vargaftik2021drive} and QUIC-FL~\cite{ben2024accelerating}. We expect similar guarantees to be derivable for other RHT-based scalar quantizers.\footnote{For a scalar quantizer, expected squared reconstruction error is a sum of marginal expected errors; this equality does not require independence. DRIVE\textquotesingle s scale and DME bias are treated separately in Section~\ref{sec:drive}.}

Additionally, we turn to the problem of Vector Quantization (VQ), where blocks of entries are jointly quantized. In VQ,
having a marginal distribution that is approximately Gaussian is insufficient, as the joint block distribution determines the fixed-codebook error. We show that this can be achieved by composing three successive RHTs.

Finally, while composing RHTs restores strict worst-case guarantees, natural data distributions are rarely perfectly adversarial. A corollary of our general statements is that if an initial vector already exhibits suitably small values for specific norms, additional RHT steps are redundant. We establish that these requirements, specifically, the cubed $\ell_3$ norm dictating scalar marginals, and the squared $\ell_\infty$ norm dictating multi-dimensional cross-correlation, can be explicitly computed in $\mathcal{O}(d)$ linear time. Thus, practical systems can dynamically check an input vector to determine whether 1, 2, or 3 RHTs suffice for the stated bounds, preserving empirical speed without sacrificing theoretical bounds.

Our presentation is intentionally modular. Specifically, we first state general results for composing two RHTs, without application-specific definitions. We then introduce the estimators and quantizers needed for DRIVE~\cite{vargaftik2021drive} and QUIC-FL~\cite{ben2024accelerating}. We take a similar approach when analyzing compositions of three RHTs. This keeps the core probabilistic statements separate from the downstream subtleties of each method, which we hope provides an easier path for future applications. 

\vspace{-1mm}
Our core contributions can be summarized as follows:
\vspace{-2mm}
\begin{itemize}%
    \item \textbf{Universal Scalar Marginals (2-RHT):} We prove that for every arbitrary input vector and every fixed coordinate, two RHTs provide a uniform $\mathcal{O}(d^{-1/2})$ normal additive approximation in both the Kolmogorov ($d_K$) and 1-Wasserstein ($W_1$) metrics. 

    \item \textbf{Application of 2-RHT Bounds to Scalar Quantizers:}
    \begin{itemize}
        \item \textbf{Formal Guarantees for DRIVE \cite{vargaftik2021drive}~(\Cref{tab:intro-drive}):} We plug these scalar $W_1$ bounds into the analysis of DRIVE, recovering the URR-based guarantees for both the biased and the unbiased estimator up to a decaying $\mathcal{O}(d^{-1/2})$ term.
        \item \textbf{Formal Guarantees for QUIC-FL \cite{ben2024accelerating}~(\Cref{tab:intro-quic}):} We apply the $d_K$ bound to restore the exact URR-based transmission bandwidth threshold up to an additive $\mathcal{O}(d^{-1/2})$ decaying term,  eliminating the $3.2\times$ multiplicative penalty of a single RHT.
    \end{itemize}

    \item \textbf{Decorrelation and VQ (3-RHT):} 
    We prove that two RHTs may fail to preserve the Gaussian expected error of a fixed codebook, whereas a composition of three RHTs provides the weak dependence required for VQ. This allows any fixed, bounded URR or Gaussian codebook to yield the same expected-error limit, up to a decaying additive term.
    
    \item \textbf{Adaptive $\mathcal{O}(d)$ Linear-Time Check:} We translate our theoretical constraints into a simple linear-time method that checks sufficient conditions for using fewer RHTs based on the $\ell_3$ and $\ell_\infty$ norms.
    
\end{itemize}

\begin{table}[t]
\centering
\small
\setlength{\tabcolsep}{2pt}
\renewcommand{\arraystretch}{1.1}
\begin{tabular*}{\linewidth}{@{\extracolsep{\fill}}lllll@{}}
\toprule
& \multicolumn{2}{c}{\textbf{(a) DRIVE's vNMSE (\S~\ref{sec:drive})}} & \multicolumn{2}{c}{\textbf{(b) QUIC-FL's guarantees (\S\ref{sec:quic})}} \\
\cmidrule(lr){2-3}\cmidrule(l){4-5}
& Biased Scale & Unbiased Scale & $(E_1,E_2,E_3,E_4$) & Sent frac. \\
\midrule
URR & $1-\frac{2}{\pi}$ & {\color{black}(\textbf{new})} $\frac{\pi}{2}-1$ & $(1.520,0.223,0.044,0.0098)$ & $p$ \\
1-RHT & $0.5$ & none & $(4.831,0.692,0.131,0.0272)$ & $3.2 {\cdot} p$ \\
2-RHT \textbf{(new)} & $1-\frac{2}{\pi}+\mathcal O(\frac{1}{\sqrt{d}})$ & $\frac{\pi}{2}-1+\mathcal O(\frac{1}{\sqrt{d}})$ & $(1.520,0.223,0.044,0.0098)+\mathcal O(\frac{1}{\sqrt{d}})$ & $p+\mathcal O(\frac{1}{\sqrt{d}})$ \\
\bottomrule
\end{tabular*}
\vspace{1mm}
\caption{
A summary of DRIVE's~\cite{vargaftik2021drive} and QUIC-FL's~\cite{ben2024accelerating} guarantees with the known URR ($\mathcal O(d^2)$ time) and 1-RHT bounds ($\mathcal O(d\log d)$ time), and the new 2-RHT ($\mathcal O(d\log d)$ time) bounds. 
{\color{black} Each $\mathcal O(d^{-1/2})$ term is an additive bound. The implicit constants may differ between entries, including components of the QUIC-FL tuple. Also, see 
Appendix~\ref{app:drive-urr-fixed-scale-sharp} for a sharper URR-based unbiased DRIVE vNMSE bound with the fixed scale.}
In (b), $E_b$, for $b\in\set{1,2,3,4}$, denotes QUIC-FL's vNMSE at bit budget $b$.
}
\vspace*{-4mm}
\label{tab:intro-applications}
\begin{subcaptiongroup}
\phantomsubcaption\label{tab:intro-drive}
\phantomsubcaption\label{tab:intro-quic}
\end{subcaptiongroup}
\end{table}

\section{Notations and definitions}
\vspace*{-2mm}

\textbf{Randomized Hadamard Transform (RHT).}
Fix $d=2^m$. Let $H$ be the Sylvester-Hadamard matrix~\cite{horadam2012hadamard}. Thus, $H\in\{\pm1\}^{d\times d}$ and $HH^\top=dI_d$. Additionally, let $\NH=\frac{1}{\sqrt d}H$. Thus, $\NH$ is orthogonal.

For $k\ge 1$, let $R_k = \NH D_k \NH D_{k-1}\cdots \NH D_1$, 
where $D_1,D_2,\dots, D_k$ are diagonal matrices with i.i.d. $
\pm 1$ Rademacher entries.
We call $R_k$ a composition of $k$ RHTs.
For convenience, we denote $R_0=I_d$. 

\textbf{Kolmogorov distance.} Given two distributions whose CDFs are $P$, $Q$, their Kolmogorov distance is defined as $d_K(P,Q) = \sup_{x\in\mathbb R} |P(x) - Q(x)|$. 

\textbf{1-Wasserstein distance} (also known as the Earth Mover’s Distance)\textbf{.} Utilizing the Kantorovich-Rubinstein duality \cite{arjovsky2017wasserstein,chen2021stein,villani2009optimal}, we use an alternative definition of the 1-Wasserstein distance given by  
$W_1(X, Y) = \sup_{f \in \text{Lip}(1)} \Big| \mathbb{E}[f(X)] - \mathbb{E}[f(Y)] \Big|$.

\textbf{Unit sphere.} We denote by $S^{d-1}$ the $(d-1)$-dimensional unit sphere, i.e., the set of all vectors $a \in \mathbb{R}^d$ such that $\norm{a}_2=1$.

\textbf{Coordinate.} For a vector $a \in \mathbb{R}^d$, we use $a_i$ and $(a)_i$ to denote its $i$'th coordinate.

\textbf{Binary sign convention.} We use $\operatorname{sign}(z)=1$ if $z\ge 0$ and $\operatorname{sign}(z)=-1$ otherwise,
applied coordinatewise to vectors.

\textbf{Normalized input} For an input $x \in \mathbb{R}^d \setminus \{0\}$,  let $\tilde{x}=\frac{x}{\norm{x}_2}\in S^{d-1}$ be its unit direction vector.

\vspace*{-2mm}

\section{Composing Two RHTs}
\vspace*{-1mm}

We start by considering the setting of composing two RHTs, under the Kolmogorov and 1-Wasserstein distance metrics. Specifically, we are interested in the properties of $R_2 x$ for any $x \in \mathbb{R}^{d}\setminus \{0\}$. 

\vspace*{-1mm}
\subsection{Kolmogorov distance with two RHTs}
\vspace*{-1mm}

Let $a \in S^{d-1}$ and define $\rho_3(a)=\sum_{i=0}^{d-1} |a_i|^3$, which satisfies  $d^{-1/2}\le\rho_3(a)\le1$. Intuitively, the $\ell_3$ norm cubed $\rho_3(a)$ measures how ``flat'' the unit vector $a$ is. If the coordinates of $a$ all have size about $d^{-1/2}$, then $\rho_3(a)$ is of order $d^{-1/2}$ while if one coordinate is large, then $\rho_3(a)$ is large; also, $\rho_3(a)\le\|a\|_\infty$. We use $\rho_3$ because it is the third-moment quantity in the Berry--Esseen bound. Thus, the analysis for two RHTs works by showing that the first RHT makes $\rho_3$ sufficiently small on average, and then that the second RHT ensures that an individual coordinate is close to a Gaussian-distributed random variable when $\rho_3(a)$ is small. We start with the following lemma. 

\begin{lemma}[Average $\ell_3$ smoothing after one RHT]\label{lem:l3}
For every $\widetilde x\in S^{d-1}$ and a random sign diagonal $D$, it holds that,
$\E\rho_3(\NH D \widetilde x) \le \frac{C_3}{\sqrt d}$ where $C_3=3^{3/4}\approx 2.2795$.
\end{lemma}

\begin{proof}
Let $a=\NH D \widetilde x$. 
Thus, $a_1 = \frac{1}{\sqrt{d}}\sum_{i=0}^{d-1} \tilde{x}_i\varepsilon_i$, where the $\varepsilon_i$ are i.i.d.\ Rademacher $\pm 1$ signs (the signs of $H$ are absorbed into the Rademacher signs). Since all coordinates have the same marginal distribution, $\E\rho_3(a)=d\,\E|a_1|^3$.
Let $M=\sum_{i=0}^{d-1} \tilde{x}_i\varepsilon_i$.
Then $\E M^2=\norm{\tilde{x}}_2^2=1$, and a direct fourth-moment calculation gives
$
\E M^4
=
\sum_{i=0}^{d-1} \tilde{x}_i^4 + 6\sum_{i<j} \tilde{x}_i^2\tilde{x}_j^2
=
3\left(\sum_{i=0}^{d-1} \tilde{x}_i^2\right)^2 - 2\sum_{i=0}^{d-1} \tilde{x}_i^4
\le 3.$

By Hölder's inequality,
$
\E|M|^3 \le (\E M^4)^{3/4}\le 3^{3/4}=C_3.
$
Therefore,
$$
\E\rho_3(a)
=
d\,\E|a_1|^3
=
d \cdot \E |\frac{1}{\sqrt{d}}M|^3
=
d\cdot d^{-3/2}\E|M|^3
=
\frac{\E|M|^3}{\sqrt d}
\le \frac{C_3}{\sqrt d}.
 \qedhere$$
\end{proof}

We also utilize a well-known Berry--Esseen bound for non-i.i.d variables derived in {\cite[Theorem 7]{tyurin2009new}} which we restate here for clarity. See also Tyurin\textquotesingle s published analysis~\cite{tyurin2011convergence}.

\textbf{Theorem 7 of \cite{tyurin2009new}}
(Berry--Esseen for non-i.i.d variables).
Let $X_0,\dots,X_{d-1}$ be independent real-valued random variables with
$\E X_i=0$ and $\sum_{i=0}^{d-1} \E X_i^2 =1$.
If $G\sim N(0,1)$, then
\[
\sup_{t\in\R}\left|\Pbb\!\left(\sum_{i=0}^{d-1} X_i\le t\right)-\Pbb(G\le t)\right|
\le 0.5606 \sum_{i=0}^{d-1} \E|X_i|^3.
\]

With these results at hand, we are now ready to prove the following theorem. 
\begin{theorem}[Kolmogorov distance with two RHTs]\label{thm:scalar-clt}
Let $x\in\R^d\setminus\{0\}$. Define
$
U(x)=\sqrt d\,{(R_2\widetilde x)_1}.
$
Then
$
d_K\bigl(U(x),G\bigr)
\le \frac{0.5606\,C_3}{\sqrt d}
\le \frac{1.28}{\sqrt d},
$
where $G\sim N(0,1)$.
Since the coordinates of $R_2x$ are identically distributed, the same estimate holds for every coordinate of $\sqrt d\,R_2\widetilde x$.
\end{theorem}

\begin{proof}

Let 
$a=\NH D_{1}\widetilde x = R_{1}\widetilde x\in S^{d-1}$.
Therefore, 
$$U(x)=\sqrt d\,{(R_2\widetilde x)_1}=\sqrt d\,{(\tilde{H}D_2R_1\widetilde x)_1}={(HD_2R_1\widetilde x)_1}=(HD_2a)_1 = \sum_{i=0}^{d-1} a_i\varepsilon_i\,\,,$$
where $\varepsilon_i=H_{1i}(D_2)_{ii}$ are i.i.d.\ Rademacher signs. By Theorem 7 of~\cite{tyurin2009new}, $d_K\bigl((U(x) \!\mid\! a), G\bigr)\le 0.5606\,\rho_3(a)$. 
Taking expectation and applying Lemma~\ref{lem:l3} yields
$$
d_K\bigl(U(x),G\bigr) \le \E_a [d_K\bigl((U(x) \mid a), G\bigr)]
\le 0.5606\,\E[\rho_3(a)] 
\le \frac{0.5606\,C_3}{\sqrt d} \,\,.\qedhere$$
\end{proof}

\subsection{1-Wasserstein distance with two RHTs}
We utilize and restate here for clarity the following lemma.

\textbf{Lemma 2.4 of \cite{chen2021stein}}($W_1$ Berry--Esseen Method\footnote{Also appears in~\cite[Section 3]{chen2010normal}.}).
Let $X_0, \dots, X_{d-1}$ be independent zero-mean random variables with $\sum_{i=0}^{d-1} \mathbb{E}[X_i^2] = 1$. Let $G \sim \mathcal{N}(0, 1)$ be a standard Gaussian. Then:
\[
W_1\left(\sum_{i=0}^{d-1} X_i, \, G\right) \le C_W \sum_{i=0}^{d-1} \mathbb{E}\big[|X_i|^3\big]
\]
where $C_W\le 3$ is an absolute constant.

Now we move to bound the 1-Wasserstein distance.

\begin{theorem}[1-Wasserstein distance with two RHTs]\label{thm:w1-clt}
Let $x\in\R^d\setminus\{0\}$. Define again
$
U(x)=\sqrt d\,{(R_2\widetilde x)_1}.
$
Then
$
W_1\bigl(U(x),G\bigr)
\le \frac{C_{W}C_3}{\sqrt{d}},
$
where $G\sim N(0,1)$.
Since the coordinates of $R_2x$ are identically distributed, the same estimate holds for every coordinate of $\sqrt d\,R_2\widetilde x$.
\end{theorem}
\begin{proof}
Again, let $a=\NH D_{1}\widetilde x= {R_{1}\widetilde x}\in S^{d-1}$ and 
 $U(x)= \sum_{i=0}^{d-1} a_i\varepsilon_i\,\,,$
where $\varepsilon_i=H_{1i}(D_2)_{ii}$ are i.i.d.\ Rademacher signs. Notice that \emph{conditioned on $a$}, the random variables $X_i = a_i\varepsilon_i$ are independent, have zero mean, and satisfy $\sum_{i=0}^{d-1} \mathbb{E}[(a_i \epsilon_i)^2] = \sum_{i=0}^{d-1} a_i^2 = 1$.

Applying Lemma 2.4 of \cite{chen2021stein} yields
$$
W_1(U(x) \mid a, G) \le C_W \sum_{i=0}^{d-1} \mathbb{E}\big[|a_i \epsilon_i|^3 \mid a\big] = C_W \sum_{i=0}^{d-1} |a_i|^3 = C_W \rho_3(a)\,.
$$
Taking expectation and applying Lemma~\ref{lem:l3} yields
\begin{align*}
W_1(U(x), G) &= \sup_{f \in \text{Lip}(1)} \Big| \mathbb{E}\big[\mathbb{E}[f(U(x)) \mid a]\big] - \mathbb{E}[f(G)] \Big| \\
&\le \E_a\big[W_1(U(x) \mid a, G)\big] \le C_W \E[\rho_3(a)] \le \frac{C_{W}C_3}{\sqrt{d}} \,\,.\qedhere
\end{align*}
\end{proof}

\vspace*{-1mm}
\section{Applying the 2-RHT results to DRIVE}\label{sec:drive}
\vspace*{-1mm}

The DRIVE algorithm~\cite{vargaftik2021drive} provides a highly efficient method for biased and unbiased 1-bit (per entry) quantization. Then, in DRIVE, the input vector $x \in \mathbb{R}^d$ is rotated to $\mathcal{R}x$ where $\mathcal{R}$ is a random rotation matrix. Then, the rotated vector is deterministically quantized to a binary vector by taking its sign. To reconstruct the vector, the estimate is $\hat{x} = S \cdot \mathcal{R}^{-1}\text{sign}(\mathcal{R}x)$, for some scalar $S \in \mathbb{R}$.

Importantly, the representation of the quantized vector includes just the sign vector. $\mathcal{R}$ is not stored and is generated independently by the dequantizer using shared randomness~\cite{ben2021send}. This is a standard assumption~\cite{chen2024ml,li2024thc,warraich2025optireduce} and can be implemented by using the same PRNG seed. The same applies when using RHT hereafter.

In the \textbf{biased} configuration, the objective is to strictly minimize the vector's Normalized Mean Squared Error ($\text{vNMSE} = \frac{\mathbb{E}[\|x - \hat{x}\|_2^2]}{\|x\|_2^2}$). DRIVE achieves this by calculating the optimal MSE-minimizing scale factor $S = \frac{1}{d}\|\mathcal{R}x\|_1$. Because the randomized rotation forces the coordinates to behave similarly to Gaussian variables in high dimensions, this biased scaling yields a vNMSE that asymptotically converges to exactly $1 - \frac{2}{\pi} \approx 0.363$ for large $d$.

In the \textbf{unbiased} configuration, the scale factor is either $S = \frac{\|x\|_2^2}{ \|\mathcal{R}x\|_1}$ or $S = \frac{\|x\|_2^2}{ \E\|\mathcal{R}x\|_1}$. Since the latter is more amenable to formal analysis and, as shown in \cite[Appendix A.4]{vargaftik2021drive}, yields a tighter vNMSE bound, we use it. Both these specific scalar corrections explicitly guarantee that the reconstructed vector is unbiased, meaning $\mathbb{E}[\hat{x}] = x$. Imposing this constraint slightly increases the individual vNMSE, which asymptotically converges for both to $\frac{\pi}{2} - 1 \approx 0.571$. 

When quantizing low-dimensional vectors, applying URRs is computationally tractable and provides rigorously bounded theoretical guarantees. However, in practical implementations, high-dimensional data overwhelmingly favor RHT. Yet, as highlighted in the analysis of the DRIVE algorithm, this substitution is not mathematically flawless; while the RHT successfully mimics a uniform rotation for many natural inputs, it can structurally fail to uniformize certain worst-case vectors (such as sparse inputs). 
In particular, DRIVE~\cite{vargaftik2021drive} gives a weaker $\vnmse\le 0.5$ bound for the biased case when using RHT instead of URR, and no bound for the unbiased RHT-based variant.

We now show that, after two RHTs, DRIVE recovers the URR vNMSE bounds for both the biased and unbiased variants, up to an additive term inversely polynomial in $d$.

\subsection{The biased configuration}

 By~\cite[Lemma~1]{vargaftik2021drive}, when using $S = \frac{1}{d}\|\mathcal{R}x\|_1$, for \emph{any} rotation matrix $\mathcal{R}$, the vNMSE is:
\begin{equation} \label{eq:nmse}
\vnmse_{\mathcal{R}} = 1 - \frac{1}{d} \mathbb{E}\Big[\|\mathcal{R} \tilde{x}\|_1^2\Big]\,\,.
\end{equation}

For a URR, as mentioned, the error converges to exactly $1 - \frac{2}{\pi} \approx 0.363$ for large $d$. Our goal is to show that a matching asymptotic upper bound holds for $\mathcal{R} \triangleq R_2$.
By Jensen's inequality:
$$
\frac{1}{d} \mathbb{E}\Big[\|R_2 \tilde{x}\|_1^2\Big] \ge \frac{1}{d} \Big(\mathbb{E}\big[\|R_2 \tilde{x}\|_1\big]\Big)^2\,\,.
$$

By linearity of expectation,
$
\mathbb{E}\big[\|R_2 \tilde{x}\|_1\big] = \sum_{i=0}^{d-1} \mathbb{E}[|(R_2 \tilde{x})_i|] = d \cdot \mathbb{E}[|(R_2 \tilde{x})_1|] \,\,.
$

Substituting this into our squared bound yields
$$
\frac{1}{d} \mathbb{E}\Big[\|R_2 \tilde{x}\|_1^2\Big] \ge \frac{1}{d} \Big(d \, \mathbb{E}[|(R_2 \tilde{x})_1|]\Big)^2 = d \cdot \big(\mathbb{E}[|(R_2 \tilde{x})_1|]\big)^2
\,\,.$$
Substituting this into Equation~\eqref{eq:nmse} yields
\begin{equation} \label{eq:nmse_bound}
\vnmse_{R_2} \le 1 - d \cdot \big(\mathbb{E}[|(R_2 \tilde{x})_1|]\big)^2 \,\,.
\end{equation}

We have now expressed a bound on $\vnmse_{R_2}$ in terms of a single coordinate. We now use Lemma \ref{lemma:u1_lb} below, which uses the fact, similar to the URR, that any single coordinate is close to a Gaussian distribution. As our lower bound can be negative for small $d$, we define $(z)_+ = \max(z, 0)$ to obtain: 
$$
d \cdot \big(\mathbb{E}[|(R_2 \tilde{x})_1|]\big)^2 \ge \bigg( \bigg(\sqrt{\frac{2}{\pi}} - \frac{C_W C_3}{\sqrt{d}}\bigg)_{\!\!+} \bigg)^2 = \frac{2}{\pi} - \mathcal{O}\left(\frac{1}{\sqrt{d}}\right)
$$

Substituting this into Equation~\eqref{eq:nmse_bound} yields the desired upper bound:
$$
\vnmse_{R_2} \le 1 - d \cdot \big(\mathbb{E}[|(R_2 \tilde{x})_1|]\big)^2 \le 1 - \bigg( \bigg(\sqrt{\frac{2}{\pi}} - \frac{C_W C_3}{\sqrt{d}}\bigg)_{\!\!+} \bigg)^2 = 1 - \frac{2}{\pi} + \mathcal{O}\left(\frac{1}{\sqrt{d}}\right)
\,\,.$$

\begin{lemma}\label{lemma:u1_lb}
    Let $x \in \mathbb{R}^d \setminus \{0\}$ be the input vector, $\tilde{x}=\frac{x}{\norm{x}_2}$ and define $U(x) = \sqrt{d} \cdot (R_2 \tilde{x})_1$. Then, 
    $$
    \mathbb{E}\big[|U(x)|\big] \ge \sqrt{\frac{2}{\pi}} - \frac{C_W C_3}{\sqrt{d}}\,\,.
    $$
\end{lemma}
\begin{proof}
For any $u,v \in \mathbb{R}$, by the reverse triangle inequality, i.e., $\big||u| - |v|\big| \le |u - v|$ , the absolute value function $f(z) = |z|$ is exactly $1$-Lipschitz. Therefore, $W_1$ bounds the $L_1$ norm difference:
\begin{equation} \label{eq:wasserstein_lip}
\Big| \mathbb{E}[|X|] - \mathbb{E}[|Y|] \Big| \le W_1(X, Y)
\end{equation}

Let $G\sim N(0,1)$. Using Equation \eqref{eq:wasserstein_lip}, rearranging, and applying Theorem \ref{thm:w1-clt} yields,
$$
\mathbb{E}[|U(x)|] \ge \mathbb{E}[|G|] - W_1(U(x), G) \ge \sqrt{\frac{2}{\pi}} -  \frac{C_{W}C_3}{\sqrt{d}} \,\,. \qedhere
$$   
\end{proof}

\subsection{The unbiased configuration}

We now analyze the bias and variance of the unbiased 2-RHT-based DRIVE estimator for an arbitrary input vector $x\in\mathbb{R}^d \setminus \{0\}$. As defined at the beginning of Section~\ref{sec:drive}, to guarantee unbiasedness under a URR, the algorithm scales the quantized vector by $S = \frac{\|x\|_2^2}{\mathbb{E}[\|\mathcal{R}_U x\|_1]}$ where $\mathcal{R}_U$ is a random rotation matrix.  This simplifies to $S = \frac{\|x\|_2}{c_d\sqrt{d}}$, where $c_d \triangleq \frac{1}{\sqrt{d}}\mathbb{E}[\|\mathcal{R}_U \tilde{x}\|_1]$ is the scaled $L_1$ norm expectation for a uniformly distributed unit vector $\mathcal{R}_U \tilde{x} \in \mathcal{S}^{d-1}$. 
Recall the estimate:
\begin{equation*}
\hat{x} = S R_2^{-1}\text{sign}(R_2x) = \frac{\|x\|_2}{c_d\sqrt{d}}R_2^{-1}\text{sign}(R_2\tilde{x}).
\end{equation*}
Let us define the expected direction vector generated by the 2-RHT as $\mu = \frac{1}{\sqrt{d}}\mathbb{E}[R_2^{-1}\text{sign}(R_2\tilde{x})]$. By factoring out the scalars, our expected estimator evaluates to $\mathbb{E}[\hat{x}] = \frac{\|x\|_2}{c_d}\mu$. The estimation bias is the vector difference between this expectation and the input vector $x$:
\begin{equation}
B(x) \triangleq \mathbb{E}[\hat{x}] - x = \frac{\|x\|_2}{c_d}(\mu - c_d\tilde{x}).
\label{eq:bias_def}
\end{equation}

With this definition at hand, we continue with the following theorem.

\begin{theorem} \label{thm:unbiased_drive} 
Let $x \in \mathbb{R}^d \setminus \{0\}$ be an arbitrary vector, and define the 2-RHT DRIVE estimator with the uniform-rotation unbiased scale by $\hat{x} = \frac{\|x\|_2}{c_d\sqrt{d}}R_2^{-1}\text{sign}(R_2x)$. As $d \to \infty$, the estimator's normalized squared bias decays to $0$, and its normalized variance converges to the same value as for a URR:
\noindent\begin{minipage}{0.48\textwidth}
\begin{equation}
\frac{\|B(x)\|_2^2}{\|x\|_2^2} \le \mathcal{O}(d^{-1/2}) \tag{i} \label{eq:i}
\end{equation}
\end{minipage}%
\hfill
\begin{minipage}{0.48\textwidth}
\begin{equation}
\frac{\mathrm{Var}(\hat{x})}{\|x\|_2^2} = \frac{\pi}{2}-1 + \mathcal{O}(d^{-1/2}) \tag{ii} \label{eq:ii}
\end{equation}
\end{minipage}
\end{theorem}

\textbf{Proof intuition.} To prove the bias decays to zero, we show that $\mu$ converges to the target vector $c_d\tilde{x}$. We achieve this by decomposing $\mu$ into a parallel component (its projection onto $\tilde{x}$) and an orthogonal component ($\mu_\perp$). 

Analyzing these components directly against the algorithm's constant $c_d$ involves cumbersome algebra. Instead, our proof leverages a simpler proxy scale: the expected absolute value of a standard Gaussian, $c \triangleq \sqrt{2/\pi}$. Because 2-RHT coordinates approximate standard Gaussians (Theorem \ref{thm:w1-clt}), the scaled expected $L_1$ norm of any unit vector under 2-RHT concentrates around $c$. We use this to bound both components:

\begin{enumerate}
    \item \textbf{Parallel component:} The scalar projection of $\mu$ onto $\tilde{x}$ is $\langle\mu,\tilde{x}\rangle=d^{-1/2}\mathbb{E}\|R_2\tilde{x}\|_1$, which is close to $c$.
    \item \textbf{Orthogonal component:} We show that the perpendicular component $\mu_\perp$ is small, and we then bound its effect on the overall bias using the Pythagorean theorem.
\end{enumerate}

Finally, because the algorithm's scale $c_d$ converges to $c$ with $|c-c_d|=\mathcal{O}(d^{-1})$, we substitute $c$ for $c_d$ at the end of the proof to establish the exact variance and bias bounds. The formal proof is deferred to Appendix~\ref{app:unbiased-drive-proof}; the required asymptotic expansion of the scaling constant $c_d$ is given in Appendix~\ref{app:cd_expansion}.

\textbf{Implications for Distributed Mean Estimation (DME).} When $N$ arbitrary vectors $x_{(c)}$ are averaged via their independent 2-RHT-based estimates $\hat{x}_{avg}=\frac{1}{N}\sum_{c=0}^{N-1}\hat{x}_{(c)}$ (each having a decaying bias), the Normalized Mean Squared Error ($\mathrm{NMSE}$) decomposes into the aggregated variance and squared bias. Because vectors use independent 2-RHT rotations, the resulting overall variance scales down by a $1/N$ factor. Concurrently, by Jensen's inequality, the squared magnitude of the aggregated bias vector is bounded by the average of the squared individual biases. Thus, the resulting error is close to the one achieved via uniform rotations:
$%
\mathrm{NMSE} \triangleq \frac{\mathbb{E}[\|\hat{x}_{avg}-x_{avg}\|_2^2]}{\frac{1}{N}\sum_{c=0}^{N-1}\|x_{(c)}\|_2^2} \le \frac{\frac{\pi}{2}-1}{N} + \mathcal{O}(d^{-1/2}) \,\,.
$%

\section{Applying the 2-RHT results to QUIC-FL}\label{sec:quic}

State-of-the-art Distributed Mean Estimation algorithms, such as QUIC-FL \cite{ben2024accelerating}, rely heavily on \emph{Bounded Support Quantization (BSQ)}.\footnote{See also~\cite{dettmers2022gpt3,dettmers2023spqr,lee2024owq}.} In BSQ, the algorithm transmits coordinates that fall outside a predefined range at baseline precision (often 32 or 16 bits) to avoid outlier errors, while stochastically quantizing the rest. Following the notation of \cite{ben2024accelerating}, coordinates outside the range  $[-t_p, t_p]$ are sent exactly (i.e., at baseline precision), where the subscript $p$ represents that $t_p$ is chosen so that $G \sim \mathcal{N}(0, 1)$ lies outside this range with probability $p$. This means that after a URR, each coordinate falls outside this range with probability approximately and asymptotically $p$. 

As discussed in \cite{ben2024accelerating}, when QUIC-FL uses a (single) RHT instead of a URR, bounds can still be proven, but they are significantly weaker. In particular, the issue is that the expected fraction of items falling outside $[-t_p, t_p]$ can be substantially more than $p$. An inequality proven in \cite{BentkusDzindzalieta} shows that for worst-case inputs, the tail probability of a coordinate obtained using a single RHT can be up to $3.1787p \approx 3.2p$;  that is, around $3.2$ times larger than that of a standard Gaussian. This multiplicative penalty in the bound on outliers cascades through the QUIC-FL analysis, leading to overestimates of both bandwidth provisioning and expected quantization error.

Our results for composing two RHTs are directly applicable and remove this penalty. Specifically, using our results for the Kolmogorov metric (Theorem \ref{thm:scalar-clt}), we show that QUIC-FL can achieve essentially the same bounds when using two composed RHTs instead of a URR, up to terms asymptotically vanishing in the dimension.  

\subsection{Expected outlier fraction and bandwidth bounds}

Recall that we let $p = \mathbb{P}(|G| > t_p)$ be the expected fraction of coordinates falling outside the range $[-t_p,t_p]$ for a standard Gaussian. With a single RHT, the expected fraction of exactly-sent coordinates (Theorem G.1 in \cite{ben2024accelerating}) can be bounded as follows: let $Z=\sqrt{d}(R_1 \tilde{x})_1$ be the first coordinate after one RHT; then
$
\mathbb{P}\big(|Z| > t_p\big) \le 3.2 \cdot \mathbb{P}\big(|G| > t_p\big) = 3.2 {\cdot}p
$. The bound on the expected number of coordinates requiring full precision is thus much higher than for URRs.

When composing two RHTs, our Kolmogorov bound (Theorem \ref{thm:scalar-clt}) guarantees that the expected fraction of outliers is asymptotically the same as for the Gaussian distribution.

\begin{theorem}[Expected Outlier Fraction with Two RHTs]\label{thm:quic_bandwidth}
As before, let $U(x) = \sqrt{d}(R_2 \tilde{x})_1$. Let $t_p > 0$ be a threshold such that $\mathbb{P}(|G| > t_p) = p$ for $G \sim \mathcal{N}(0, 1)$. Then, the probability that a coordinate falls outside $[-t_p, t_p]$ after two RHTs is bounded by:
\[
\mathbb{P}\big(|U(x)| > t_p\big) \le p+\frac{2\cdot1.28}{\sqrt d}= p + \frac{2.56}{\sqrt{d}} = p + \mathcal{O}\big(d^{-1/2}\big) \,\,.
\]
\end{theorem}
\begin{proof}
Because the coordinates of $\sqrt{d}(R_2 \tilde{x})$ are identically distributed, the expected fraction of outliers equals the probability that a single coordinate exceeds the threshold $t_p$. We rewrite this tail probability in terms of the cumulative distribution function and apply the Kolmogorov bound from Theorem \ref{thm:scalar-clt}. Let $F_U(t) = \mathbb{P}\big(U(x) \le t\big)$ and $F_G(t) = \mathbb{P}\big(G \le t\big)$ denote the cumulative distribution functions of $U(x)$ and $G$, respectively. Then:
\begin{align*}
\mathbb{P}\big(|U(x)| > t_p\big)
&= \mathbb{P}\big(U(x) > t_p\big) + \mathbb{P}\big(U(x) < -t_p\big) 
\le 1 - F_U(t_p) + F_U(-t_p) \\
&\le 1 - F_G(t_p) + F_G(-t_p) + 2 d_K(U(x), G) 
= \mathbb{P}\big(|G| > t_p\big) + 2 d_K(U(x), G).
\end{align*}
Since $d_K(U(x), G) \le \frac{1.28}{\sqrt{d}}$ by Theorem \ref{thm:scalar-clt}, substituting this bound yields the desired result.
\end{proof}
As $d$ grows, the expected fraction of outliers converges to $p$, saving $\approx 70$\% of the bandwidth overhead if using the bound above for a single RHT. 

\subsection{Quantization error and bounded total variation}\label{sec:quic:error}
To bound the actual quantization error (vNMSE) of BSQ, we define the expected squared error function conditioned on the coordinate value: $e(z) = \mathbb{E}[(\hat{z} - z)^2 \mid z]$. For $|z| \le t_p$, $e(z)$ is defined by the stochastic quantizer. For $|z| > t_p$, the values are sent exactly, meaning $e(z) = 0$.

This creates a jump discontinuity at the thresholds $\pm t_p$. Because the error drops to zero, the function is non-Lipschitz. Consequently, the 1-Wasserstein metric used in Section \ref{sec:drive} cannot be applied to bound this expected error. Instead, we use the Kolmogorov distance ($d_K$), which can bound the expectation of any function with bounded total variation ($TV$). For a function with jump discontinuities, the total variation is appropriately defined by $TV(f) = \sup_{n\ge1}\sup_{z_0 < \dots < z_n} \sum_{j=1}^n |f(z_j) - f(z_{j-1})| < \infty$.

\begin{theorem}[Quantization Error with Two RHTs]\label{thm:quic_error}
Let $e:\mathbb{R} \to \mathbb{R}$ be right-continuous, of bounded variation, with finite limits at $\pm\infty$ (representing the expected squared error function for BSQ). Let $TV(e)$ denote its total variation. For any input $x \in \mathbb{R}^d \setminus \{0\}$, the expected quantization error of a scaled transformed coordinate $U(x) = \sqrt{d}(R_2 \tilde{x})_1$ after two RHTs satisfies:
\[
\Big| \mathbb{E}[e(U(x))] - \mathbb{E}[e(G)] \Big| \le TV(e) d_K(U(x), G) \le \frac{1.28 \cdot TV(e)}{\sqrt{d}} \,\,, 
\]
where $G \sim \mathcal{N}(0, 1)$. 
\end{theorem}

\textbf{Proof intuition.} According to Theorem \ref{thm:scalar-clt}, a scaled transformed coordinate $U(x)=\sqrt d (R_2\tilde x)_1$ is close to Gaussian in Kolmogorov distance. Since Kolmogorov distance controls how much probability the two distributions assign below any threshold, and since $TV(e)$ measures the total amount by which the quantization-error function can change, the expected value of $e$ can change by at most their product. The complete proof is deferred to Appendix \ref{app:quic_proofs}.

\textbf{Refined worst-case error bounds.} In \cite[Theorem G.3]{ben2024accelerating}, computing worst-case bounds for the quantization error $E_b$ required partitioning the support into coarse intervals (e.g., $[0, 1.5]$) and penalizing the entire probability mass of that interval with the maximal error occurring at its boundary. This Riemann sum overestimation was necessary due to the 1-RHT multiplicative tail bound, leading to inflated theoretical bounds (e.g., producing 1-bit quantization error $E_1 \le 4.831$ as $d \to \infty$).

By substituting 1-RHT with 2-RHT, the theoretical vNMSE bounds converge to the continuous values (identified numerically in \cite[Section 3.5]{ben2024accelerating}).
Transitioning from 1-RHT to 2-RHT thus recovers the desired theoretical error limits up to a decaying $\mathcal{O}(d^{-1/2})$ term.

\begin{table}[t]
\centering
\begin{tabular}{|c|c|c|c|}
\hline
\textbf{Bit Budget} & \textbf{1-RHT Bound} (Thm G.3) & \textbf{2-RHT Bound} & \textbf{Improvement} \\ \hline
$b=1$ & $E_1 \le 4.831$ & $E_1 \approx 1.520$ & $\mathbf{3.17\times}$ \textbf{tighter} \\ \hline
$b=2$ & $E_2 \le 0.692$ & $E_2  \approx 0.223$ & $\mathbf{3.10\times}$ \textbf{tighter} \\ \hline
$b=3$ & $E_3 \le 0.131$ & $E_3 \approx 0.044$ & $\mathbf{2.97\times}$ \textbf{tighter} \\ \hline
$b=4$ & $E_4 \le 0.0272$ & $E_4 \approx 0.0098$ & $\mathbf{2.77\times}$ \textbf{tighter} \\ \hline
\end{tabular}
\vspace{1mm}
\caption{Comparing the $b$-bit vNMSE, i.e., $E_b$, as a function of the bit-budget $b$ for 1-RHT and 2-RHT, adapting results from \cite{ben2024accelerating}.}
\end{table}

\section{Beyond scalar quantization: Three RHTs for VQ}\label{sec:3rht}

While two-RHTs are sufficient to ensure the one-dimensional marginal distributions required for \emph{scalar} rotation-based quantization algorithms like DRIVE and QUIC-FL, some compression schemes require multi-dimensional VQ in which coordinates are partitioned into sets of size $k$ and treated as $ k$-dimensional vectors. This is particularly prominent in frameworks built for inference, KV cache quantization, and vector databases, e.g.,~\cite{gao2024rabitq,malinovskii2025higgs,quartet2026}. 

Vector quantizers group multi-dimensional data into localized regions based on the nearest representative point. For standard codebooks to be efficient, the coordinates within a $k$-sized block must have a suitable joint distribution, as Gaussian marginals alone do not suffice. For example, if coordinates are highly correlated, the data cluster in a lower-dimensional subspace, leading to significant quantization error as the distribution differs substantially from the multivariate Gaussian distribution the codebook is designed for.

In this section, we show that two RHTs may not preserve the Gaussian expected error of a fixed codebook. However, we establish that adding a third RHT yields a weak coordinate correlation sufficient for using Gaussian codebooks while maintaining their theoretical performance. Here, we consider Hadamard matrices constructed via the standard Sylvester's construction. %

\subsection{The correlation bottleneck of Two RHTs}

To understand why two RHTs may be insufficient for VQ, consider the conditional covariance between two coordinates of the transformed vector.

\begin{lemma}[2-RHT Correlation Bottleneck]\label{lem:2rht_corr}
There exist input vectors $\tilde{x} \in S^{d-1}$ and distinct coordinates $i\neq j$ such that, for
$U=\sqrt d\,R_2\tilde{x}$ and $a=\NH D_1\tilde{x}$, for every realization of the \mbox{sign matrix $D_1$,}
\[
\Cov_{D_2}(U_i,U_j\mid a)=\pm 1\ .
\]
Thus, the two coordinates are perfectly conditionally correlated or anti-correlated.
\end{lemma}

\textbf{Proof intuition.} When evaluating the covariance between two output coordinates $U_i$ and $U_j$ (conditioned on the intermediate vector after the first RHT), the properties of the Hadamard matrix collapse the sum into pairwise products of the original input coordinates. If the input is highly sparse (e.g., only two non-zero coordinates), this sum reduces to a single surviving term. Because the marginal variance of each coordinate is 1, a conditional covariance of $\pm 1$ dictates that the coordinates are perfectly conditionally correlated or anti-correlated. The derivation is deferred to Appendix \ref{sec:lem:2rht_corr}.

The distinction between unconditional and conditional covariance is essential. For $\tilde{x}=(e_0+e_1)/\sqrt2$, the unconditional law of $(U_0,U_1)$ converges to $(G,SG)$, where $G\sim\mathcal N(0,1)$ and $S$ is an independent Rademacher sign. This pair has zero covariance but strongly dependent coordinates. For the (optimal) fixed codebook $\mathcal C=\{(\pm2/\sqrt\pi,0),(0,\pm2/\sqrt\pi)\}$, its limiting expected error exceeds the Gaussian error, as derived in Appendix~\ref{app:2rht-error-gap}. Importantly, this example shows failure of codebook universality rather than guaranteed improvement for a given codebook.

Accordingly, we next show that a 3-RHT sequence bounds the conditional covariance itself in RMS, ruling out such bad cases for fixed coordinate pairs.

\subsection{The 3-RHT conditional decorrelation guarantee}

By using three RHTs, we can eliminate the conditional correlation patterns observed for two RHTs.

\begin{theorem}[3-RHT Decorrelation]\label{thm:3rht_decorr}
Let $\tilde{x} \in S^{d-1}$, $y = R_1 \tilde{x}$, $U = \sqrt{d} R_3 \tilde{x}$, and for $i \neq j$ define the conditional covariance
$C_{i,j}(y, D_2) = \Cov_{D_3}(U_i, U_j \mid y, D_2)$.

{\color{black}
Then, for every fixed $y=R_1\tilde x$,
$
\E_{D_2}\big[C_{i,j}(y,D_2)\big]=0
$
and
$
\E_{D_2}\big[C_{i,j}(y,D_2)^2\big]
\le
2\|y\|_\infty^2.
$ 
Moreover, the Root Mean Square (RMS) of this $D_3$-conditional covariance over $D_1, D_2$ decays:
$$
\left(
\E_{D_1,D_2}\big[
C_{i,j}(R_1\tilde{x},D_2)^2
\big]
\right)^{1/2}
\le
\mathcal{O} \left(\sqrt{\frac{1}{d}}\right) \,\,.
$$
}
\end{theorem}

{\color{black}
\textbf{Proof intuition.} While the 2-RHT may fail, e.g., on sparse inputs, the additional RHT lets us average the conditional covariance over two independent sign matrices. For any fixed output $y$ of the first RHT, we express the conditional covariance between two coordinates after the final RHT as a signed sum over disjoint coordinate pairs. This implies zero mean over the middle sign matrix, and also gives a deterministic second-moment bound in terms of $\|y\|_\infty^2$. To obtain the worst-case rate, we average the same pairwise expression over the first RHT. A fourth-moment calculation for two distinct coordinates gives a fixed pair contribution of order $\frac{1}{d^2}$, and there are $\frac{d}{2}$ such pairs. This yields an RMS conditional covariance bound of order $\frac{1}{\sqrt d}$. The complete proof is given in Appendix \ref{sec:lem:3rht_decorr}.
}

\subsection{Codebook universality: Matching the URR limit}

In $k$-dimensional VQ, a codebook $\mathcal{C} = \{c_1, \dots, c_M\} \subset \mathbb{R}^k$ maps each $k$-dimensional block of the vector to its nearest centroid. The expected quantization error $\mathcal{E}(\mathcal{C}, V)$ is the expected squared distance between the block $V$ and its selected centroid from the codebook $\mathcal{C}$.

For a URR, any $k$-dimensional block converges to a standard multivariate Gaussian $Z \sim \mathcal{N}(0, I_k)$ as $d \to \infty$. Thus, the optimal codebook is derived via $k$-means clustering over this distribution. 
We show that a 3-RHT achieves the same quantization error up to a decaying term.

\begin{theorem}[3-RHT Codebook Universality]\label{thm:3rht_vq_error}
Let $\tilde{x}\in S^{d-1}$ and $U = \sqrt{d} R_3 \tilde{x}$. Fix the block size $k$ and a finite codebook $\mathcal{C} \subset \mathbb{R}^k$ with radius $B = \max_{c \in \mathcal{C}} \|c\|_2$. 

Then, the difference in expected quantization error satisfies:
\[
\Big| \E\Big[\min_{c \in \mathcal{C}} \|U_{0:k-1} - c\|_2^2\Big] - \E\Big[\min_{c \in \mathcal{C}} \|Z - c\|_2^2\Big] \Big| 
\le 
\mathcal{O}_{k, B}\bigg(\frac1{\sqrt d}\bigg) \,\,.
\]
\end{theorem}
\textbf{Proof intuition.} We utilize the multivariate generalization of Stein's Method to bound the 1-Wasserstein distance between the 3-RHT block and a standard multivariate Gaussian. This distance is governed by two components: the deviation of the block's covariance matrix from the identity \(I_k\), and its expected third absolute moments. Our decorrelation guarantee (Theorem \ref{thm:3rht_decorr}) bounds the covariance mismatch, while the 1-RHT smoothing property (Lemma \ref{lem:l3}) bounds the third moments.
The only additional issue is that the VQ error function for vector $v$,  $\min_{c\in\mathcal C}\|v-c\|_2^2$
is quadratic and therefore not globally Lipschitz. Thus, we separate the quadratic part from the codebook-dependent part via
$
\min_{c\in\mathcal C}\|v-c\|_2^2
=
\|v\|_2^2
+
\min_{c\in\mathcal C}
\bigl(\|c\|_2^2-2\langle v,c\rangle\bigr).
$
The second term is globally \(2B\)-Lipschitz, and the first term cancels in expectation because both the 3-RHT block and \(Z\sim\mathcal N(0,I_k)\) have expected squared norm \(k\). The formal proof is deferred to Appendix \ref{sec:lem:3rht_vq_error}.

\section{Adaptive linear-time verification}
\label{sec:adaptive}
Because practical data distributions are often not adversarial, we can dynamically adapt the required number of RHTs. Our proofs show that the number of required RHT layers is governed by two quantities of the normalized input 
$
    \tilde{x}=\frac{x}{\|x\|_2}.
$
For the scalar marginal normal approximation, the relevant quantity is the cubed
$\ell_3$ mass
$
    \rho_3(\tilde{x})
    =
    \sum_{r=0}^{d-1}|\tilde{x}_r|^3.
$

{\color{black}
For deciding whether the preliminary RHT in the VQ construction can be skipped,
the relevant quantity is the squared norm $\ell_\infty$, i.e., 
$
    \|\tilde{x}\|_\infty^2.
$
Thus, before applying the worst-case number of RHTs, a system can first
check whether the input is already sufficiently flat.
}

\paragraph{Scalar check.}
If
$
    \rho_3(\tilde{x})\le \eta_3,
$
then the one-coordinate Berry--Esseen step used in
Theorems~\ref{thm:scalar-clt} and~\ref{thm:w1-clt} already applies to
a single RHT. Namely, for every fixed coordinate
$
    U_i=\sqrt d\,(\NH D\tilde{x})_i,
$
we have
$
    d_K(U_i,G)\le 0.5606\,\eta_3,
    \,\,
    W_1(U_i,G)\le C_W\eta_3,
$
where $G\sim\mathcal N(0,1)$. Therefore, if
$
    \eta_3=\mathcal O(d^{-1/2}),
$
then 1-RHT already gives the same asymptotic guarantees as the
worst-case 2-RHT theorem. 

{\color{black}
\paragraph{Vector check.}

For VQ, scalar marginal normality is insufficient; the relevant quantity is the conditional covariance between coordinates in a fixed block. The worst-case 3-RHT guarantee in Theorem~\ref{thm:3rht_decorr} gives an $O(\frac{1}{\sqrt d})$ RMS conditional covariance bound. If we want to skip the preliminary smoothing RHT and use only two RHTs, then we apply the fixed-input part of Theorem~\ref{thm:3rht_decorr} directly to the current input. More precisely, let
$
    U=\sqrt d\,\NH D_b\NH D_a\tilde{x}
$
be the 2-stage RHT of \(\tilde{x}\). For distinct coordinates \(i\neq j\), define the conditional covariance
$
    C_{ij}(\tilde{x},D_a)
    =
    \Cov_{D_b}\!\left(U_i,U_j \mid \tilde{x},D_a\right).
$
The fixed-input bound in Theorem~\ref{thm:3rht_decorr}, applied with \(y=\tilde{x}\), gives
$
    \E_{D_a}\!\left[C_{ij}(\tilde{x},D_a)^2\right]
    \le
    2\|\tilde{x}\|_\infty^2.
$
Equivalently, if \(\|\tilde{x}\|_\infty^2\le \eta_\infty\), then
$
    \left(
    \E_{D_a}\!\left[C_{ij}(\tilde{x},D_a)^2\right]
    \right)^{1/2}
    \le
    \sqrt{2\eta_\infty}.
$

Thus, the adaptive check transfers into the VQ guarantee: using two RHTs on such an input gives a covariance contribution of order \(\sqrt{\eta_\infty}\), and hence a VQ error gap of order
$
    \mathcal O_{k,B}\!\left(\sqrt{\eta_\infty}+\frac1{\sqrt d}\right).
$

In particular, if \(\eta_\infty=\mathcal O(\frac{1}{d})\), then two RHTs match the worst-case 3-RHT bound \(\mathcal O_{k,B}(\frac{1}{\sqrt d})\). If instead \(\eta_\infty=\mathcal O(\frac{\log d}{d})\), then two RHTs give a weaker but still vanishing guarantee
$
    \mathcal O_{k,B}\!\left(\sqrt{\frac{\log d}{d}}\right).
$
}

\paragraph{Implementation.}
Both checks can be evaluated in one pass over $x$. Define
$S_2=\sum_{r=0}^{d-1}x_r^2$, $S_3=\sum_{r=0}^{d-1}|x_r|^3$, $M_2=\max_{0\le r<d}x_r^2$.
For $x\neq 0$,
$
    \rho_3(\tilde{x})
    =
    \frac{S_3}{S_2^{3/2}},
    \,\,
    \|\tilde{x}\|_\infty^2
    =
    \frac{M_2}{S_2}.
$
Therefore, the adaptive GPU-friendly verification step costs $\mathcal O(d)$
time and $\mathcal O(1)$ space. %

{\color{black}
In summary, scalar quantization methods such as \cite{ben2024accelerating, vargaftik2021drive} can use the $\rho_3(\tilde{x})$ check to decide whether one RHT is already sufficient. VQ methods can use the $\|\tilde{x}\|_\infty^2$ check to decide whether two RHTs already provide the desired covariance bound with vanishing-term guarantee determined by the input's $\ell_\infty$ norm.

} 

\section{Empirical validation}\label{sec:empirical}
We test the released HIGGS~\cite{malinovskii2025higgs} $p=2$ codebook at 3 bits per value on inputs with 32 consecutive nonzeros. Three RHTs improve vNMSE over one and two RHTs (Table~\ref{tab:higgs-summary}). The comparison to Uniform Random Rotation (URR) shows confidence intervals that include zero in the three dimensions with URR results, indicating that the difference is not statistically significant. The full details appear in Appendix~\ref{app:Fixed-codebook HIGGS experiment}.
\begin{table}[t]
\centering\small
\begin{tabular}{rrr}
\toprule
$d$ & Improvement over one RHT & Improvement over two RHTs \\
\midrule
128 & $8.59\%$ & $1.77\%$ \\
512 & $7.68\%$ & $2.67\%$ \\
2048 & $7.15\%$ & $2.69\%$ \\
8192 & $6.74\%$ & $3.10\%$ \\
\bottomrule
\end{tabular}
\caption{Relative vNMSE reductions using three RHTs with the fixed HIGGS codebook. }
\label{tab:higgs-summary}
\end{table}

At $d=512$, the calibrated scalar/vector checks use a second/third RHT for $9.38\%$/$12.1\%$ of synthetic mixed-workload groups. Adaptive vNMSE is $0.03437$/$0.02949$, versus $0.03434$/$0.02949$ for always-two/three, successfully skipping unneeded RHTs. On 1968 groups from 21 Qwen2.5-0.5B tensors, including weights, activations, KV-cache blocks, and gradients, the checks skip the extra RHT for $81.7\%$/$30.3\%$ of groups. Appendix~\ref{app:Adaptive scalar and vector checks} gives the complete details.

Appendix~\ref{app:scalar-marginal-checks} compares scalar Gaussian approximation after one, two, and three RHTs using Kolmogorov and 1-Wasserstein distances.

\long\def\annexperiments{%
\begingroup
\definecolor{fullrhtpurple}{RGB}{128,0,128}
\color{black}
\section{Approximate Nearest Neighbor Search}
\label{sec:ann-extra-rht}
\paragraph{Setup.}
We evaluate RaBitQ~\cite{gao2024rabitq} after one or two full RHTs on the six datasets in Table~\ref{tab:ann-input-stats}. For each dataset, we use a single center $c$, computed as the mean of the first 4096 training vectors, and zero-pad the residuals to the next power of two. We use RaBitQ's native per-residual scale selection at 2, 3, and 4 bits per transformed coordinate. The experiments use five random seeds, with the same $D_1$ diagonal for the 1-RHT and 2-RHT variants.

\begin{table}[t]
\centering
\color{black}
\begin{tabular}{lrr}
\toprule
Dataset & Size & Dimension \\
\midrule
MSong & $994{,}185$ & $420$ \\
Image & $2{,}340{,}373$ & $150$ \\
GIST & $1{,}000{,}000$ & $960$ \\
Word2Vec & $1{,}000{,}000$ & $300$ \\
SIFT & $1{,}000{,}000$ & $128$ \\
Deep1M & $1{,}000{,}000$ & $256$ \\
\bottomrule
\end{tabular}
\captionsetup{position=bottom,skip=7pt}
\caption{ANN datasets. Dimensions are reported before zero-padding to the next power of $2$.}
\label{tab:ann-input-stats}
\end{table}

For each query $q$, we rank all database vectors using the distance estimate
\[
\widehat d(q,x)=\|q-c\|_2^2+\|x-c\|_2^2-2\langle R(q-c),\widehat y_x\rangle,
\]
where $R$ is the composition of the RHTs and $\widehat y_x$ is RaBitQ's decoded representation of $R(x-c)$ at its native reconstruction scale. Recall@$100$ is the fraction of the exact 100 nearest neighbors recovered among the 100 vectors with the smallest estimated distances. We average recall over queries and then over the five seeds.

\paragraph{Results.}
Figure~\ref{fig:full-rht-20261022-rabitq-primary} reports the change in Recall@$100$ from one to two RHTs. At 2 bits, two RHTs improve mean recall by $4.55$ percentage points on MSong and $2.93$ percentage points on Image. These gains depend on the bit budget: at 3 bits, MSong's recall decreases by approximately $0.70$ percentage points with two RHTs. The changes on GIST, SIFT, and Deep1M are small. Thus, additional mixing can improve recall, but its benefit depends on the dataset and quantization regime.

\begin{figure*}[t]
\centering
\includegraphics[width=\textwidth]{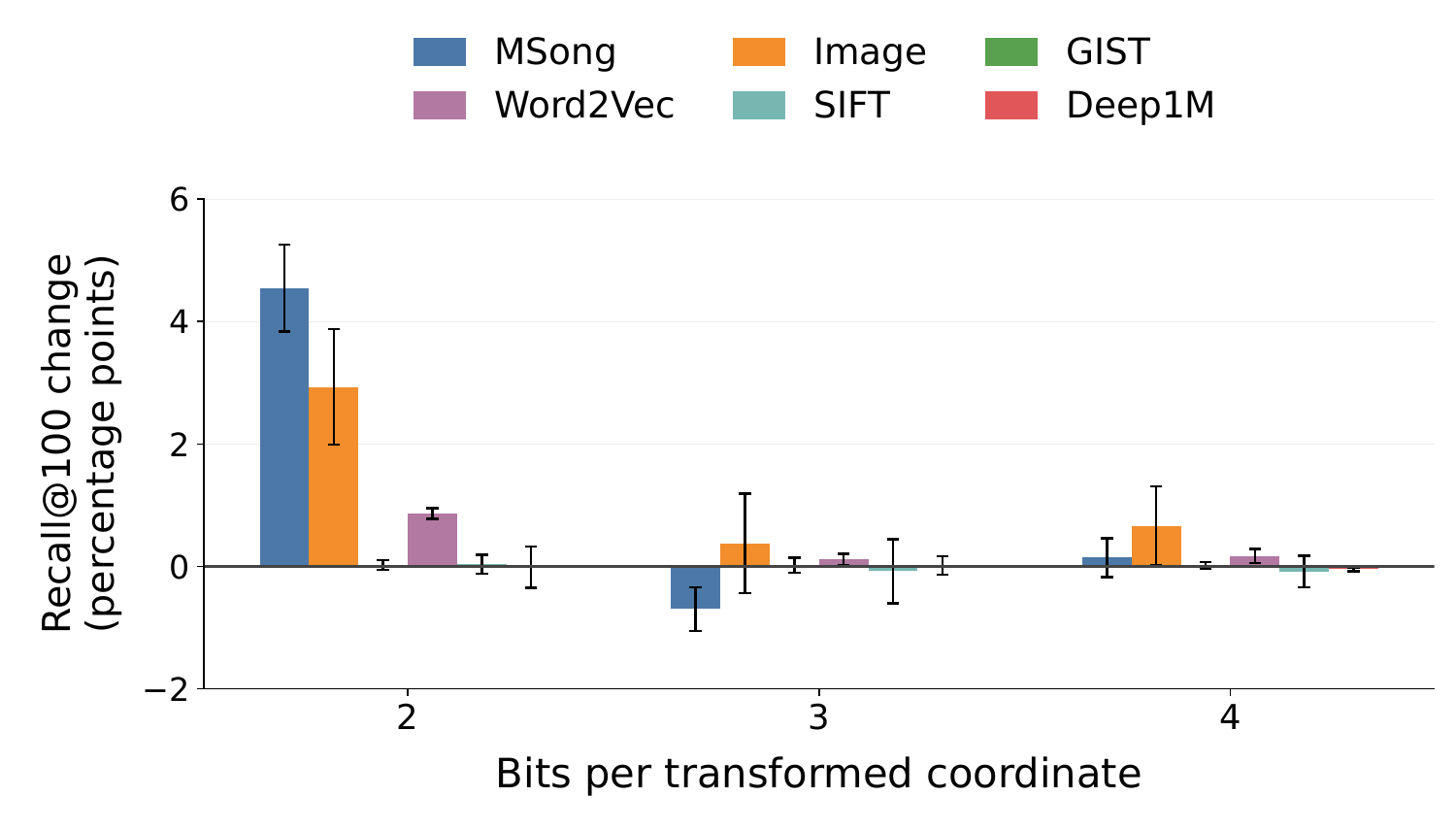}
\caption{Change in RaBitQ Recall@$100$ for two full RHTs relative to one RHT, in percentage points, at 2--4 bits per transformed coordinate. Each group contains one bar per dataset. Bars show mean differences over five paired random seeds; error bars show 95\% Student-$t$ confidence intervals computed from the paired differences, with the database and queries fixed. Positive values indicate higher recall.}
\label{fig:full-rht-20261022-rabitq-primary}
\end{figure*}

\endgroup
}
\annexperiments

\section{Conclusion}
The transition from dense URRs to fast, structured transformations like the Randomized Hadamard Transform (RHT) is a fundamental system-level optimization for modern quantization frameworks. Historically, however, this transition has carried a heavy theoretical penalty for worst-case inputs, yielding weak bounds and reducing theoretically optimal algorithms to practical heuristics. In this work, we demonstrated that this penalty is not an inherent limitation of structured matrices, but simply a consequence of using an insufficient number of transforms. 

By composing two RHTs, we established universal $\mathcal{O}(d^{-1/2})$ convergence bounds in both the Kolmogorov and 1-Wasserstein metrics for any arbitrary input vector. We demonstrated that these scalar guarantees seamlessly plug into the analyses of state-of-the-art scalar quantization schemes. By doing so, we successfully recovered the exact URR-based guarantees for DRIVE and the Bounded Support Quantization (BSQ) error integrals for QUIC-FL, without the prohibitive computational cost. 

Moving beyond scalar quantization, we identified a critical conditional correlation bottleneck on sparse inputs that can prevent two RHTs from matching Gaussian expected error. We resolved this by proving that a 3-RHT sequence successfully decorrelates block coordinates, recovering the expected-error limit of fixed Gaussian VQ codebooks via Stein's Method. Crucially, to ensure that these robust worst-case guarantees do not burden average-case execution, we established that simple linear-time $\mathcal{O}(d)$ checks on the input's $\ell_3$ and $\ell_\infty$ norms can dynamically dictate when one, two, or three RHTs suffice at runtime.

Our modular probabilistic framework opens several promising directions for future research. For example, while three RHTs are sufficient to guarantee weak dependence for standard fixed-size VQ blocks, exploring a four-RHT (4-RHT) composition presents a natural extension. Specifically, a 4-RHT sequence could provide enhanced independence structures through a fourth RHT step, which may be necessary to formally guarantee decorrelation across non-fixed-size blocks required by next-generation adaptive quantizers.

\bibliographystyle{plain}

\appendix

\section{Deferred analysis for Section 4.2}
\label{app:unbiased_proofs}

In this appendix, we first establish the asymptotic expansion of the scaling constant $c_d$ used by the unbiased DRIVE algorithm, and then provide the formal proof of Theorem \ref{thm:unbiased_drive}.

\subsection{Asymptotic expansion of the scaling constant $c_d$.}
\label{app:cd_expansion}

We now formalize the convergence gap between the scaling constant $c_d$ and the standard Gaussian absolute expectation $c = \sqrt{\frac{2}{\pi}}$, where $g \sim \mathcal{N}(0,1)$.

For a URR matrix $\mathcal{R}_U$, define $\tilde{x}_U \triangleq \mathcal{R}_U \tilde{x}$. Then, $\tilde{x}_U \in \mathcal{S}^{d-1}$ is a point on the unit sphere, distributed uniformly at random. By \cite[Lemma 9]{vargaftik2021drive}, its expected $L_1$ norm evaluates via the Beta function to:
\begin{equation*}
\mathbb{E}[\|\tilde{x}_U\|_1] = \frac{2d}{(d-1) \cdot B\left(\frac{1}{2}, \frac{d-1}{2}\right)} \,\,.
\end{equation*}
Substituting this into $c_d = \frac{1}{\sqrt{d}} \mathbb{E}[\|\tilde{x}_U\|_1]$, we rewrite the Beta function in terms of the Gamma function, $B(x,y) = \frac{\Gamma(x)\Gamma(y)}{\Gamma(x+y)}$. Evaluating $\Gamma(1/2) = \sqrt{\pi}$ yields:
\begin{equation*}
c_d = \frac{2\sqrt{d}}{d-1} \frac{\Gamma(d/2)}{\sqrt{\pi} \Gamma((d-1)/2)} \,\,.
\end{equation*}
Applying the Gamma function recurrence relation $\Gamma(z+1) = z\Gamma(z)$ to the denominator provides $(d-1)\Gamma((d-1)/2) = 2\Gamma((d+1)/2)$. Substituting this simplifies the expression for $c_d$ to:
\begin{equation}
c_d = \sqrt{\frac{d}{\pi}} \frac{\Gamma(d/2)}{\Gamma((d+1)/2)}.
\label{eq:cd_gamma}
\end{equation}

To bound the resulting ratio of Gamma functions, we apply the Tricomi-Erdélyi expansion \cite[Eq. 6.1.47, page 257]{abramowitz1964handbook}. For $z \to \infty$ and bounded constants $a$ and $b$, the expansion is stated strictly using Big-O notation as:
\begin{equation*}
\frac{\Gamma(z+a)}{\Gamma(z+b)} = z^{a-b} \left( 1 + \frac{(a-b)(a+b-1)}{2z} + \mathcal{O}\left(\frac{1}{z^2}\right) \right).
\end{equation*}
Letting $z = d/2$, $a=0$, and $b=1/2$, we obtain:
\begin{align*}
\frac{\Gamma(d/2)}{\Gamma((d+1)/2)} &= \left(\frac{d}{2}\right)^{-1/2} \left( 1 + \frac{(-1/2)(-1/2)}{2(d/2)} + \mathcal{O}\left(\frac{1}{d^2}\right) \right) \\
&= \sqrt{\frac{2}{d}} \left( 1 + \frac{1}{4d} + \mathcal{O}\left(\frac{1}{d^2}\right) \right).
\end{align*}
Applying this exact expansion back into Equation \eqref{eq:cd_gamma} gives:
\begin{equation*}
c_d = \sqrt{\frac{d}{\pi}} \sqrt{\frac{2}{d}} \left( 1 + \frac{1}{4d} + \mathcal{O}\left(\frac{1}{d^2}\right) \right) = \sqrt{\frac{2}{\pi}} \left( 1 + \frac{1}{4d} + \mathcal{O}\left(\frac{1}{d^2}\right) \right).
\end{equation*}

Finally, using $c = \sqrt{2/\pi}$:
\begin{align*}
| c - c_d | = \left| \sqrt{\frac{2}{\pi}} - \sqrt{\frac{2}{\pi}} \left( 1 + \frac{1}{4d} + \mathcal{O}\left(\frac{1}{d^2}\right) \right) \right| = \frac{1}{2d\sqrt{2\pi}} + \mathcal{O}\left(\frac{1}{d^2}\right).
\end{align*}
Thus, the scaling discrepancy $| c - c_d |$ is bounded by $\mathcal{O}(d^{-1})$.

{\color{black}\subsection{URR unbiased fixed-scale vNMSE for DRIVE}
\label{app:drive-urr-fixed-scale-sharp}

The following non-asymptotic calculation is used for DRIVE's unbiased URR entry of
Table~\ref{tab:intro-drive} and improves over the $\frac{\pi}{2}-1+\mathcal{O}(\sqrt\frac{\log d}{d})$ bound derived in \cite{vargaftik2021drive}. Recall the unbiased fixed-scale
DRIVE estimator, namely the variant that uses
$
    S
    =
    \frac{\|x\|_2^2}{\E\|\mathcal R_Ux\|_1}
    =
    \frac{\|x\|_2}{c_d\sqrt d},
$
where
$
    c_d
    =
    \frac1{\sqrt d}\E\|\mathcal R_U\tilde{x}\|_1.
$

\begin{theorem}[URR DRIVE unbiased fixed-scale vNMSE]
\label{thm:drive_urr_fixed_scale_sharp}
Let \(x\in\mathbb R^d\setminus\{0\}\), let \(\tilde{x}=x/\|x\|_2\), and let
\(\mathcal R_U\) be a uniform random rotation. Consider DRIVE with the unbiased fixed scale
\[
    S
    =
    \frac{\|x\|_2^2}{\E\|\mathcal R_Ux\|_1}
    =
    \frac{\|x\|_2}{c_d\sqrt d},
    \qquad
    \hat{x}
    =
    S\mathcal R_U^{-1}\operatorname{sign}(\mathcal R_Ux).
\]
Then
\[
    \vnmse
    =
    \frac{\E\|\hat{x}-x\|_2^2}{\|x\|_2^2}
    =
    \frac1{c_d^2}-1.
\]
Moreover, for every \(d\ge 1\),
\[
    \vnmse
    \le
    \frac{\pi}{2}-1.
\]
\end{theorem}

\begin{proof}
For clarity, we divide the proof into three steps.

\vspace{0.5em}
\textbf{Step 1: Reducing the vNMSE to \(c_d\).}
By the unbiasedness result for DRIVE's fixed-scale estimator
\cite[Appendix~A.4]{vargaftik2021drive}, we have
$
    \E[\hat{x}]=x.
$
Therefore,
\[
\begin{aligned}
    \E\|\hat{x}-x\|_2^2
    &=
    \E\|\hat{x}\|_2^2
    -
    2\langle \E[\hat{x}],x\rangle
    +
    \|x\|_2^2  
    =
    \E\|\hat{x}\|_2^2
    -
    \|x\|_2^2.
\end{aligned}
\]
It remains to compute \(\|\hat{x}\|_2^2\). Since \(\mathcal R_U\) is orthogonal,
\[
\begin{aligned}
    \|\hat{x}\|_2^2
    &=
    \left\|
        \frac{\|x\|_2}{c_d\sqrt d}
        \mathcal R_U^{-1}\operatorname{sign}(\mathcal R_Ux)
    \right\|_2^2  
    =
    \frac{\|x\|_2^2}{c_d^2d}
    \left\|\operatorname{sign}(\mathcal R_Ux)\right\|_2^2.
\end{aligned}
\]
The vector \(\operatorname{sign}(\mathcal R_Ux)\) has \(d\) entries, each of
magnitude \(1\). Hence
$
    \left\|\operatorname{sign}(\mathcal R_Ux)\right\|_2^2=d,
$
and so
$
    \|\hat{x}\|_2^2
    =
    \frac{\|x\|_2^2}{c_d^2}.
$
Thus
$
    \E\|\hat{x}-x\|_2^2
    =
    \left(\frac1{c_d^2}-1\right)\|x\|_2^2.
$
Dividing by \(\|x\|_2^2\) gives
$
    \vnmse
    =
    \frac1{c_d^2}-1.
$

\vspace{0.5em}
\textbf{Step 2: Lower bounding \(c_d\).}
From Appendix~\ref{app:cd_expansion},
\[
    c_d
    =
    \sqrt{\frac d\pi}
    \frac{\Gamma(d/2)}{\Gamma((d+1)/2)}.
\]
We now show that \(c_d\ge \sqrt{2/\pi}\). 
By Gautschi's inequality \cite[Eq.~5.6.4]{dlmf}, originally due to~\cite{gautschi1959gamma},
for \(x>0\) and \(s\in(0,1)\),
\[
    x^{1-s}
    <
    \frac{\Gamma(x+1)}{\Gamma(x+s)}.
\]
We apply this with
$
    x=\frac d2,
    \,
    s=\frac12.
$
Then
\[
    \left(\frac d2\right)^{1/2}
    <
    \frac{\Gamma(d/2+1)}{\Gamma((d+1)/2)}.
\]
Using the Gamma recurrence relation
$
    \Gamma(d/2+1)=\frac d2\,\Gamma(d/2),
$
we get
\[
    \left(\frac d2\right)^{1/2}
    <
    \frac d2
    \frac{\Gamma(d/2)}{\Gamma((d+1)/2)}.
\]
Dividing by \(d/2\) gives
\[
    \frac{\Gamma(d/2)}{\Gamma((d+1)/2)}
    >
    \sqrt{\frac2d}.
\]
Substituting this into the formula for \(c_d\),
\[
    c_d
    =
    \sqrt{\frac d\pi}
    \frac{\Gamma(d/2)}{\Gamma((d+1)/2)}
    >
    \sqrt{\frac d\pi}\sqrt{\frac2d}
    =
    \sqrt{\frac2\pi}.
\]
In particular,
$
    c_d^2\ge \frac2\pi.
$

\vspace{0.5em}
\textbf{Step 3: Concluding the vNMSE upper bound.}
From Step 1,
$
    \vnmse
    =
    \frac1{c_d^2}-1.
$
From Step 2,
$
    c_d^2\ge \frac2\pi.
$
Therefore,
$
    \frac1{c_d^2}
    \le
    \frac\pi2.
$
Substituting into the exact vNMSE identity yields
\[
    \vnmse
    =
    \frac1{c_d^2}-1
    \le
    \frac\pi2-1.
\]
This concludes the proof. \qedhere
\end{proof}}

\subsection{Proof of Theorem \ref{thm:unbiased_drive}}\label{app:unbiased-drive-proof}

\begin{proof}

For clarity, we divide the proof into five steps.

\textbf{Step 1: Proxy scale closeness.} 
Let $c = \sqrt{\frac{2}{\pi}}$. By Theorem \ref{thm:w1-clt} and the fact that the absolute value function is 1-Lipschitz, the expected absolute value of any individual 2-RHT coordinate differs from $c$ by at most $\mathcal{O}(d^{-1/2})$. Recall that, for any arbitrary unit vector $v \in \mathcal{S}^{d-1}$, all $d$ coordinates of $R_2 v$, share the same marginal distribution. Thus, its scaled expected $L_1$ norm evaluates exactly to this single-coordinate absolute expectation. Thus, we establish a universal uniform bound $\delta$ defining the closeness of the true 2-RHT expectation to our Gaussian proxy value $c$:
\begin{equation}
\delta \triangleq \sup_{v \in \mathcal{S}^{d-1}} \left| \frac{1}{\sqrt{d}}\mathbb{E}[\|R_2 v\|_1] - c \right| = \mathcal{O}(d^{-1/2}).
\label{eq:delta_bound}
\end{equation}
This guarantees that for \emph{any} unit direction, the expected $L_1$ norm after 2-RHT is in $[c-\delta, c+\delta]$.

\textbf{Step 2: The parallel component.} 
We decompose the expected direction $\mu = \frac{1}{\sqrt{d}}\mathbb{E}[R_2^{-1}\text{sign}(R_2\tilde{x})]$ into its parallel projection along $\tilde{x}$ and an orthogonal error vector $\mu_\perp$. 
For the parallel component, using $R_2^{-1}=R_2^T$, the identity $\langle R_2^\top a, b\rangle = \langle a, R_2 b\rangle$ and the property $\langle \text{sign}(y), y\rangle = \|y\|_1$, we have:
\begin{equation*}
\langle\mu,\tilde{x}\rangle = \frac{1}{\sqrt{d}}\mathbb{E}[\langle \text{sign}(R_2\tilde{x}), R_2\tilde{x}\rangle] = \frac{1}{\sqrt{d}}\mathbb{E}[\|R_2\tilde{x}\|_1] \,\,.
\end{equation*}
By applying Equation \eqref{eq:delta_bound}, the parallel projection is bounded, that is: $\langle\mu,\tilde{x}\rangle \in [c-\delta, c+\delta]$.

\textbf{Step 3: The orthogonal component.}
For the orthogonal component $\mu_\perp$, let $u$ be the unit vector pointing in its direction, ensuring $\langle\tilde{x},u\rangle=0$ and $\langle\mu,u\rangle=\|\mu_\perp\|_2$. Let $t > 0$ be a positive scalar and consider the perturbed vector $R_2(\tilde{x} + tu)$. The absolute value inequality dictates, 
$$|A+B| \ge |A| + \text{sign}(A) \cdot B \,\,.$$ 
Applying this pointwise to every coordinate of the perturbed vector and summing them yields:
\begin{equation*}
\|R_2(\tilde{x} + tu)\|_1 \ge \|R_2\tilde{x}\|_1 + \langle \text{sign}(R_2\tilde{x}), R_2 tu \rangle = \|R_2\tilde{x}\|_1 + t \langle \text{sign}(R_2\tilde{x}), R_2 u \rangle.
\end{equation*}
Taking expectations, multiplying by $\frac{1}{\sqrt{d}}$, and applying the identity, 
$$\mathbb{E}[\langle \text{sign}(R_2 \tilde{x}), R_2 u \rangle] = \mathbb{E}[\langle R_2^\top \text{sign}(R_2\tilde{x}), u \rangle] = \langle \sqrt{d}\mu, u \rangle \,\,,$$ 
bounds the orthogonal magnitude:
\begin{equation}
\frac{1}{\sqrt{d}}\mathbb{E}[\|R_2(\tilde{x} + tu)\|_1] \ge \frac{1}{\sqrt{d}}\mathbb{E}[\|R_2\tilde{x}\|_1] + t\langle\mu,u\rangle \,\,.
\label{eq:orthogonal_ineq}
\end{equation}
By the Pythagorean theorem, the length of the perturbed input vector is $\|\tilde{x} + tu\|_2 = \sqrt{1+t^2}$. Because the $L_1$ norm expectation scales linearly with vector length, applying \eqref{eq:delta_bound} upper-bounds the left side by $(c+\delta)\sqrt{1+t^2}$. We lower-bound the right side using $c-\delta + t\|\mu_\perp\|_2$. Substituting these limits into \eqref{eq:orthogonal_ineq} yields:
\begin{equation*}
(c+\delta)\sqrt{1+t^2} \ge c-\delta + t\|\mu_\perp\|_2.
\end{equation*}
Using the standard scalar upper bound $\sqrt{1+t^2} \le 1 + \frac{t^2}{2}$, expanding the left side, and subtracting $c-\delta$ results in:
\begin{equation*}
2\delta + \frac{c+\delta}{2}t^2 \ge t\|\mu_\perp\|_2.
\end{equation*}
If $\|\mu_\perp\|_2 = 0$, the bound holds trivially. Otherwise, dividing by $t>0$ and setting $t=2\sqrt{\frac{\delta}{c+\delta}}$ yields:
\begin{equation*}
\|\mu_\perp\|_2 \le 2\sqrt{\delta(c+\delta)} = \mathcal{O}(\sqrt{\delta}) = \mathcal{O}(d^{-1/4}).
\end{equation*}

\textbf{Step 4: Bridging to the algorithm's scale.} 
By the Pythagorean theorem, the squared Euclidean distance from $\mu$ to the proxy target $c\tilde{x}$ is bounded by:
\begin{equation*}
\|\mu - c\tilde{x}\|_2^2 = (\langle\mu,\tilde{x}\rangle - c)^2 + \|\mu_\perp\|_2^2 \le \delta^2 + \mathcal{O}(\delta) = \mathcal{O}(\delta) = \mathcal{O}(d^{-1/2}).
\end{equation*}
The algorithm, however, uses $c_d$. As derived in Appendix \ref{app:cd_expansion}, $c_d$ converges to $c$ by $|c_d - c| = \mathcal{O}(d^{-1})$. Thus, applying the triangle inequality yields:
\begin{equation*}
\|\mu - c_d\tilde{x}\|_2 \le \|\mu - c\tilde{x}\|_2 + \|(c - c_d)\tilde{x}\|_2 \le \mathcal{O}(d^{-1/4}) + \mathcal{O}(d^{-1}) = \mathcal{O}(d^{-1/4}).
\end{equation*}
Squaring this yields $\|\mu - c_d\tilde{x}\|_2^2 \le \mathcal{O}(d^{-1/2})$. 

Recall that $\|B(x)\|_2^2 = \frac{\|x\|_2^2}{c_d^2}\|\mu - c_d\tilde{x}\|_2^2$. 
Dividing by $\|x\|_2^2$ proves Equation (\ref{eq:i}):
\begin{equation*}
\frac{\|B(x)\|_2^2}{\|x\|_2^2} = \frac{1}{c_d^2} \|\mu - c_d\tilde{x}\|_2^2 \le \mathcal{O}(d^{-1/2}) \,\,.
\end{equation*}

\textbf{Step 5: Computing the variance.} 
The relative variance we evaluate is $\frac{\mathrm{Var}(\hat{x})}{\|x\|_2^2} = \frac{\mathbb{E}[\|\hat{x}\|_2^2] - \|\mathbb{E}[\hat{x}]\|_2^2}{\|x\|_2^2}$.
Because the 2-RHT matrix preserves lengths due to its orthogonality, the uncentered $\ell_2$ norm of any $d$-dimensional $\pm 1$ sign vector is identically $\sqrt{d}$. Thus, the uncentered squared norm is deterministic:
\begin{equation*}
\mathbb{E}[\|\hat{x}\|_2^2] = S^2 d = \left(\frac{\|x\|_2}{c_d\sqrt{d}}\right)^2 d = \frac{\|x\|_2^2}{c_d^2}.
\end{equation*}
Dividing by $\|x\|_2^2$, the relative uncentered squared norm is exactly $1/c_d^2$. For the squared mean, we evaluate $\frac{\|\mathbb{E}[\hat{x}]\|_2^2}{\|x\|_2^2} = \frac{1}{c_d^2} \|\mu\|_2^2$. Since $\|\mu\|_2^2 = \langle\mu,\tilde{x}\rangle^2 + \|\mu_\perp\|_2^2 = (c \pm \delta)^2 + \mathcal{O}(d^{-1/2}) = c^2 \pm \mathcal{O}(d^{-1/2})$, we subtract this from the relative uncentered norm:
\begin{equation*}
\frac{\mathrm{Var}(\hat{x})}{\|x\|_2^2} = \frac{1}{c_d^2} - \frac{1}{c_d^2}(c^2 \pm \mathcal{O}(d^{-1/2})) = \frac{1 - c^2}{c_d^2} \pm \mathcal{O}(d^{-1/2}).
\end{equation*}
As established, $c_d = c + \mathcal{O}(d^{-1})$, which directly implies $\frac{1}{c_d^2} = \frac{1}{c^2} + \mathcal{O}(d^{-1})$. We substitute this into the leading term, absorbing the $\mathcal{O}(d^{-1})$ difference into the $\mathcal{O}(d^{-1/2})$ error. Plugging in $c = \sqrt{2/\pi}$  evaluates to $\frac{1-c^2}{c^2} = \frac{1}{c^2} - 1 = \frac{\pi}{2} - 1$, proving Equation (\ref{eq:ii}):
\begin{equation*}
\frac{\mathrm{Var}(\hat{x})}{\|x\|_2^2} = \frac{\pi}{2} - 1 + \mathcal{O}(d^{-1/2}) \,\,. \qedhere
\end{equation*}
\end{proof}

\section{Deferred analysis for Section \ref{sec:quic:error}}
\label{app:quic_proofs}

\subsection*{Proof of Theorem \ref{thm:quic_error}}\label{app:quic_proofs:thm7}

\begin{proof}
Let $F_U$ be the CDF of $U(x)=\sqrt{d}(R_2\tilde{x})_1$, and write $F_U(z^-)=\mathbb{P}(U(x)<z)$. Let $F_G$ be the CDF of the standard Gaussian $G\sim\mathcal{N}(0,1)$, so $F_G(z)=\mathbb{P}(G\le z)=\mathbb{P}(G<z)$.
Let $de$ denote the signed Stieltjes measure induced by the right-continuous bounded-variation function $e$.%

We formulate the difference in the expected quantization error given by $e(z) = \mathbb{E}[(\hat{z} - z)^2 \mid z]$ between $U(x)$ and $G$ using the Lebesgue-Stieltjes integral over the real line:
\begin{align*}
\Big| \mathbb{E}[e(U(x))] - \mathbb{E}[e(G)] \Big| &= \left| \int_{-\infty}^\infty e(z) \, dF_U(z) - \int_{-\infty}^\infty e(z) \, dF_G(z) \right| \\
&= \left| \int_{-\infty}^\infty e(z) \, d\big(F_U(z) - F_G(z)\big) \right| \,\,.
\end{align*}

To isolate the CDF difference, we apply integration by parts for Lebesgue-Stieltjes integrals:
\[
\int_{-\infty}^\infty e(z) \, d\big(F_U(z) - F_G(z)\big)
=
\Big[ e(z) \big(F_U(z) - F_G(z)\big) \Big]_{-\infty}^\infty
-
\int_{-\infty}^\infty \big(F_U(z^-) - F_G(z^-)\big) \, de(z) \,\,.
\]

We first evaluate the boundary terms. By the definition of Bounded Support Quantization, coordinates outside $[-t_p, t_p]$ are sent without quantization error, meaning $e(z) = 0$ for all $|z| > t_p$. Consequently, the boundary evaluations vanish:
\[
\lim_{z \to \infty} e(z) \big(F_U(z) - F_G(z)\big) = 0 \quad \text{and} \quad \lim_{z \to -\infty} e(z) \big(F_U(z) - F_G(z)\big) = 0 \,\,.
\]

This leaves only the integral term. Taking the absolute value yields:
\begin{align*}
\Big| \mathbb{E}[e(U(x))] - \mathbb{E}[e(G)] \Big|
&= \left| - \int_{-\infty}^\infty \big(F_U(z^-) - F_G(z^-)\big) \, de(z) \right| \\
&\le \int_{-\infty}^\infty \big| F_U(z^-) - F_G(z^-) \big| \, |de|(z) \,\,.
\end{align*}

By definition, the Kolmogorov distance uniformly bounds the absolute difference between the CDFs at any real point:
$\sup_{z} |F_U(z) - F_G(z)| = d_K(U(x), G)$.
The same bound holds for $F_U(z^-) - F_G(z^-)$. Therefore, we pull this supremum outside the integral:
\[
\int_{-\infty}^\infty \big| F_U(z^-) - F_G(z^-) \big| \, |de|(z)
\le
d_K(U(x), G) \int_{-\infty}^\infty |de|(z) \,\,.
\]
The remaining integral term, $\int_{-\infty}^\infty |de|(z)$, corresponds to the Total Variation of the error function $e(z)$, denoted $TV(e)$. To properly account for the jump discontinuities at $\pm t_p$, the Total Variation is defined by the supremum over partitions:
\[
TV(e) = \sup_{n\ge1}\sup_{z_0 < \dots < z_n} \sum_{j=1}^n \big|e(z_j) - e(z_{j-1})\big| \,\,.
\]
If $e(z)$ is piecewise absolutely continuous, this equals the sum of the integrals of $|e'(z)|$ on the smooth pieces plus the magnitudes of all jump discontinuities.

Substituting $TV(e)$ and applying the 2-RHT Kolmogorov bound derived in Theorem \ref{thm:scalar-clt} gives the desired bound:
\[
\Big| \mathbb{E}[e(U(x))] - \mathbb{E}[e(G)] \Big| \le TV(e) \, d_K(U(x), G) \le \frac{1.28 \cdot TV(e)}{\sqrt{d}} \,\,. \qedhere
\]
\end{proof}

\section{Deferred analysis for Section \ref{sec:3rht}}
\label{app:vq_proofs}

\subsection{Proof of Lemma \ref{lem:2rht_corr}}\label{sec:lem:2rht_corr}

\begin{proof}

For clarity, we divide the proof into six steps.

\textbf{Step 1: Definitions.} Let $\tilde{x} \in S^{d-1}$ and $U = \sqrt{d} R_2 \tilde{x} = \sqrt{d} \NH D_2 \NH D_1 \tilde{x}$. Define the intermediate vector $a = \NH D_1 \tilde{x}$. 
Recalling $\sqrt{d} \NH = H$, we can write $U = H D_2 a$. Conditioned on $a$, the randomness of $D_1$ is fixed, and the $i$-th coordinate of $U$ is given by $U_i = \sum_{r=0}^{d-1} H_{ir} \varepsilon_r a_r$.

\vspace{0.5em}
\textbf{Step 2: Formulating the covariance.} We seek the conditional covariance between two distinct coordinates $U_i$ and $U_j$ ($i \neq j$) over the randomness of $D_2$. Since $\E[\varepsilon_r] = 0$, the expected value $\E[U_i \mid a] = 0$, making the covariance equal to $\E[U_i U_j \mid a]$:
\[
\Cov(U_i, U_j \mid a) = \E[U_i U_j \mid a] = \E \left[ \left(\sum_r H_{ir} \varepsilon_r a_r\right) \left(\sum_s H_{js} \varepsilon_s a_s\right) \right]\,\,.
\]
Because the signs $\varepsilon$ are independent Rademacher variables, $\E[\varepsilon_r \varepsilon_s] = 0$ for $r \neq s$, and $\E[\varepsilon_r^2] = 1$. The double sum thereby collapses to:
\[
\Cov(U_i, U_j \mid a) = \sum_{r=0}^{d-1} H_{ir} H_{jr} a_r^2\,\,.
\]

\vspace{0.5em}
\textbf{Step 3: The XOR property.}
A fundamental property of the Sylvester-Hadamard matrix (using 0-based indexing) is that the product of two elements in the same column is given by the bitwise XOR of their row indices: $H_{ir} H_{jr} = H_{i \oplus j, r}$. Substituting this into our covariance yields:
\[
\Cov(U_i, U_j \mid a) = \sum_{r=0}^{d-1} H_{i \oplus j, r} a_r^2\,\,.
\]

\vspace{0.5em}
\textbf{Step 4: Expanding the intermediate vector $a$.}
Recall $a_r = \frac{1}{\sqrt{d}} \sum_{l=0}^{d-1} H_{rl} s_l \tilde{x}_l$. Squaring yields:
\[
a_r^2 = \frac{1}{d} \sum_{l=0}^{d-1} \sum_{m=0}^{d-1} H_{rl} H_{rm} s_l s_m \tilde{x}_l \tilde{x}_m\,\,.
\]
Separating the diagonal ($l=m$) from the cross-terms ($l \neq m$) and applying the XOR property $H_{rl}H_{rm} = H_{r, l \oplus m}$, we get:
\[
a_r^2 = \frac{1}{d} \sum_{l=0}^{d-1} \tilde{x}_l^2 + \frac{1}{d} \sum_{l \neq m} H_{r, l \oplus m} s_l s_m \tilde{x}_l \tilde{x}_m\,\,.
\]

\vspace{0.5em}
\textbf{Step 5: Canceling the independent sum.}
We substitute $a_r^2$ back into the covariance equation from Step 3. 
For the first term, because $i \neq j$, we know $i \oplus j \neq 0$. The sum of any non-zero row of a Hadamard matrix is $0$ ($\sum_r H_{i \oplus j, r} = 0$), perfectly eliminating the independent $\frac{1}{d}\sum \tilde{x}_l^2$ term. 
For the cross-terms, swapping the order of summation and using that $H$ is symmetric yields:
\[
\Cov(U_i, U_j \mid a) = \frac{1}{d} \sum_{l \neq m} s_l s_m \tilde{x}_l \tilde{x}_m \left( \sum_{r=0}^{d-1} H_{i \oplus j, r} H_{l \oplus m, r} \right)\,\,.
\]
The inner sum is the dot product of two Hadamard rows, which equals $d$ if $i \oplus j = l \oplus m$, and $0$ otherwise. This cancels the $\frac{1}{d}$ and acts as an indicator function $\mathbb{I}(i \oplus j = l \oplus m)$, simplifying the covariance to:
\[
\Cov(U_i, U_j \mid a) = 2 \sum_{l < m \,:\, l \oplus m = i \oplus j} s_l s_m \tilde{x}_l \tilde{x}_m\,\,.
\]

\vspace{0.5em}
\textbf{Step 6: The sparse input counterexample.}
Consider $i=0$, $j=1$, and
\[
\tilde{x}=\frac{e_0+e_1}{\sqrt 2}\in S^{d-1}.
\]
Then $i\oplus j=1$. In the covariance formula from Step 5, the only nonzero pair satisfying
$l<m$ and $l\oplus m=i\oplus j=1$ is $(l,m)=(0,1)$. Therefore,
\[
\Cov_{D_2}(U_0,U_1\mid a)
=
2s_0s_1\tilde{x}_0\tilde{x}_1
=
s_0s_1
\in\{\pm1\}.
\]
Moreover,
\[
\Var_{D_2}(U_0\mid a)=\sum_{r=0}^{d-1}H_{0r}^2a_r^2=\|a\|_2^2=1
\]
and similarly $\Var_{D_2}(U_1\mid a)=1$. Hence, the conditional Pearson correlation is exactly $\pm1$. \qedhere

\end{proof}

\subsection{Unconditional error gap for two RHTs}\label{app:2rht-error-gap}

In Lemma~\ref{lem:2rht_corr}, we computed the covariance conditioned on the first RHT. Here, we take the expectation over both $D_1$ and $D_2$ and compare the expected quantization error with that of two independent Gaussian coordinates, using the same codebook. We show that the covariance is zero, but the coordinates remain dependent, which changes the expected quantization error.

\paragraph{Step 1: The first two coordinates.}
Consider the same normalized input $\tilde{x}=(e_0+e_1)/\sqrt2$ as in Appendix~\ref{sec:lem:2rht_corr}, and let $U=\sqrt d\,R_2\tilde{x}$. Write $D_1=\diag(s_0,\ldots,s_{d-1})$ and let $S=s_0s_1$. Since $s_0,s_1$ are independent Rademacher signs, $S$ is also a Rademacher sign. The Sylvester-Hadamard matrix satisfies $H_{r0}=1$ and $H_{r1}=H_{1r}$. Thus, the intermediate vector $a=\NH D_1\tilde{x}$ satisfies
\[
a_r=(\NH D_1\tilde{x})_r
=\frac{s_0+H_{r1}s_1}{\sqrt{2d}}
=\begin{cases}
s_0\sqrt{2/d},& H_{1r}=S,\\
0,& H_{1r}=-S.
\end{cases}
\]
Row $1$ contains $d/2$ entries of each sign. Therefore, $a$ has exactly $d/2$ nonzero coordinates, at the indices $r$ for which $H_{1r}=S$.

Let $\varepsilon_r$ be the independent signs in $D_2$. Since $U=HD_2a$, $H_{0r}=1$, and only the coordinates with $H_{1r}=S$ contribute to the sums, we get
\[
U_0=s_0\sqrt{\frac2d}\sum_{r:H_{1r}=S}\varepsilon_r,
\qquad
U_1=s_0\sqrt{\frac2d}\sum_{r:H_{1r}=S}H_{1r}\varepsilon_r
=S U_0.
\]
Conditioned on $D_1$, the $d/2$ signs $s_0\varepsilon_r$ in the sum for $U_0$ are independent Rademacher variables. Their distribution is the same for every choice of $D_1$. Thus, $U_0$ is independent of $S$, and $(U_0,U_1)$ has the same distribution as $(T_d,ST_d)$, where
\[
T_d=\sqrt{\frac2d}\sum_{r=1}^{d/2}\xi_r.
\]
Here, $\xi_1,\ldots,\xi_{d/2}$ are independent Rademacher signs, chosen independently of $S$.

\paragraph{Step 2: The limiting distribution.}

By the central limit theorem, $T_d\Rightarrow G\sim\mathcal N(0,1)$. Since $\E T_d^2=1$ for every $d$, the expected absolute values converge as well:
\[
\E|T_d|\longrightarrow\E|G|=\sqrt{\frac2\pi}.
\]
Since $T_d$ and $S$ are independent, $(U_0,U_1)\Rightarrow(G,SG)$, with $G$ and $S$ independent. The limiting coordinates have zero covariance, since $\E[SG^2]=0$, but are dependent because $|G|=|SG|$.

\paragraph{Step 3: Quantization error with two RHTs.}
Set $\alpha=2/\sqrt\pi$ and consider the same four-point codebook as in Section~\ref{sec:3rht},
\[
\mathcal C=\{(\alpha,0),(-\alpha,0),(0,\alpha),(0,-\alpha)\}.
\]
For $v=(v_1,v_2)$, all four centroids have squared norm $\alpha^2$. Expanding the squared distance gives $\|v-c\|_2^2=\|v\|_2^2+\alpha^2-2\langle v,c\rangle$. Only the inner product depends on the choice of $c$. It is largest when we choose the coordinate of $v$ with the larger absolute value and use the same sign. Therefore,
\[
\begin{aligned}
e_{\mathcal C}(v)
&:=\min_{c\in\mathcal C}\|v-c\|_2^2\\
&=\|v\|_2^2+\alpha^2-2\alpha\max(|v_1|,|v_2|).
\end{aligned}
\]
For $(T_d,ST_d)$, the maximum is just $|T_d|$ and the squared norm is $2T_d^2$. Using $\E T_d^2=1$ and Step 2:
\[
\begin{aligned}
\E e_{\mathcal C}(U_0,U_1)
&=2+\frac4\pi-\frac4{\sqrt\pi}\E|T_d|,\\
\lim_{d\to\infty}\E e_{\mathcal C}(U_0,U_1)
&=2+\frac4\pi-\frac{4\sqrt2}{\pi}.
\end{aligned}
\]
The squared norm has expectation $2$ for every $d$, so we only needed to compute the limit of $\E|T_d|$.

\paragraph{Step 4: Comparison with Gaussian coordinates.}
Let $Z=(Z_1,Z_2)\sim\mathcal N(0,I_2)$. Its expected squared norm is also $2$, but the expected maximum of its absolute coordinates is larger. To compute that maximum, use the identity
\[
\max(|z_1|,|z_2|)
=\frac{|z_1+z_2|+|z_1-z_2|}{2}.
\]
Both $Z_1+Z_2$ and $Z_1-Z_2$ have distribution $\mathcal N(0,2)$, so each has expected absolute value $2/\sqrt\pi$. Taking expectations therefore gives
\[
\E\max(|Z_1|,|Z_2|)=\frac2{\sqrt\pi},
\qquad
\E e_{\mathcal C}(Z)=2+\frac4\pi-\frac8\pi=2-\frac4\pi.
\]
Subtracting the Gaussian error from the limiting two-RHT error gives a positive gap that does not decrease with $d$:
\[
\lim_{d\to\infty}\E e_{\mathcal C}(U_0,U_1)
-\E e_{\mathcal C}(Z)
=\frac{4(2-\sqrt2)}{\pi}>0.
\]
Thus, even though each coordinate converges to a standard Gaussian and the covariance is zero, the expected quantization error after two RHTs does not converge to the Gaussian error for this codebook. This example shows that two RHTs cannot guarantee the same expected error for every fixed codebook. 

\subsection{Proof of Theorem \ref{thm:3rht_decorr}}\label{sec:lem:3rht_decorr}
\begin{proof}
For clarity, we divide the proof into five steps.

\textbf{Step 1: Reduction to the 2-RHT case.}
Consider $U = \sqrt{d} R_3 \tilde{x} = \sqrt{d} \NH D_3 \NH D_2 (\NH D_1 \tilde{x})$. 
Let $y = \NH D_1 \tilde{x} = R_1 \tilde{x}$ be the output of the first RHT. If we condition on $y$, the final two RHTs $\sqrt{d} \NH D_3 \NH D_2 y$ match the structure of the 2-RHT transformation analyzed in Lemma \ref{lem:2rht_corr}.
Therefore, we can apply the previous conditional covariance formula, replacing the original input $\tilde{x}$ with $y$, and using $s$ to denote the random signs of the intermediate matrix $D_2$:
\[
C_{i,j}(y, D_2) = 2 \sum_{l < m \,:\, l \oplus m = i \oplus j} s_l s_m y_l y_m \,\,.
\]

\vspace{0.5em}
\textbf{Step 2: The expected covariance is zero.}
We analyze the expected value of this conditional covariance over the randomness of $D_2$. The sum iterates over indices where $l < m$, where $l \neq m$. Because the Rademacher signs of $D_2$ are independent and zero-mean, $\E_{D_2}[s_l s_m] = \E[s_l]\E[s_m] = 0$. Due to the linearity of expectation, the expression evaluates to zero for any given $y$: $\E_{D_2}[C_{i,j}(y, D_2)] = 0$.

\vspace{0.5em}
\textbf{Step 3: Variance over disjoint pairs.}
To bound the magnitude of the conditional correlation, we compute its variance over $D_2$. Since the mean is zero, for any given $y$, the variance is 
$$\Var_{D_2}(C_{i,j}(y, D_2)) = \E_{D_2}[C_{i,j}(y, D_2)^2]\,\,.$$
When expanding the squared sum, we generate squared terms and cross-terms. The condition $l \oplus m = i \oplus j$ partitions the $d$ indices into exactly $d/2$ uniquely matched disjoint pairs. Because the pairs are entirely disjoint, the cross-terms consist of independent Rademacher products which vanish under expectation. We are left  with the expected values of the squared terms. Since $(s_l s_m)^2 = 1$:
\[
\E_{D_2}[C_{i,j}(y, D_2)^2] = 4 \sum_{l < m \,:\, l \oplus m = i \oplus j} y_l^2 y_m^2 \,\,.
\]

\vspace{0.5em}
\textbf{Step 4: Bounding with the maximum norm.}
Let $\alpha=i\oplus j$. Since $i\neq j$, we have $\alpha\neq 0$.
We can replace the strict $l<m$ summation by a sum over all
$l\in\{0,\dots,d-1\}$, with $m=l\oplus\alpha$ (which is equivalent to $l\oplus m=\alpha$). 
Since $\alpha\neq 0$, the map $l\mapsto l\oplus\alpha$ pairs each index with a distinct index. Therefore,
the sum over all $l$ counts each unordered pair twice:
$$
\sum_{l=0}^{d-1} y_l^2 y_{l\oplus\alpha}^2
=
2\sum_{l<m:\,l\oplus m=\alpha} y_l^2y_m^2.
$$

Using the expression from Step 3, $\alpha=i\oplus j$ and the above yields,
$$
\E_{D_2}[C_{i,j}(y,D_2)^2]
=
4\sum_{l<m:\,l\oplus m=\alpha} y_l^2y_m^2 = 2\sum_{l=0}^{d-1} y_l^2 y_{l\oplus\alpha}^2\,\,.
$$

We bound this sum using the maximum squared element of $y$, denoted by
$\|y\|_\infty^2$:
$$
\begin{aligned}
\E_{D_2}[C_{i,j}(y,D_2)^2]
=
2\sum_{l=0}^{d-1} y_l^2 y_{l\oplus\alpha}^2 
\le
2\max_k(y_k^2)\sum_{l=0}^{d-1}y_l^2 
=
2\|y\|_\infty^2\sum_{l=0}^{d-1}y_l^2 \,\,.
\end{aligned}
$$

Because the first RHT preserves the $\ell_2$ norm of the input vector
$\tilde{x}\in S^{d-1}$, $\|y\|_2^2=1$. This yields
$$
\E_{D_2}[C_{i,j}(y,D_2)^2]
\le
2\|y\|_\infty^2 \,\,.
$$

{\color{black}
\vspace{0.5em}
\textbf{Step 5: Averaging directly over the first RHT.}
We now use the exact expression from Step 3 and average it directly over the signs of $D_1$. For fixed distinct $p,q$, write
\[
y_p
=
\frac1{\sqrt d}\sum_{r=0}^{d-1}H_{pr}\sigma_r\tilde{x}_r,
\qquad
y_q
=
\frac1{\sqrt d}\sum_{r=0}^{d-1}H_{qr}\sigma_r\tilde{x}_r,
\]
where $\sigma_r$ are the Rademacher signs of $D_1$. A direct fourth-moment expansion gives
\[
\E_{D_1}\!\left[
y_p^2y_q^2
\right]
=
\frac1{d^2}
\left(
1+2\theta_{pq}^2-2\sum_{r=0}^{d-1}\tilde{x}_r^4
\right),
\]
where
$
\theta_{pq}
=
\sum_{r=0}^{d-1}H_{pr}H_{qr}\tilde{x}_r^2.
$
Since $|\theta_{pq}|\le 1$, this implies
$
\E_{D_1}\!\left[
y_p^2y_q^2
\right]
\le
\frac3{d^2}.
$
Using the $d/2$ unordered pairs satisfying $p\oplus q=i\oplus j$, Step 3 yields,
$
\E_{D_1,D_2}\big[
C_{i,j}(R_1\tilde{x},D_2)^2
\big]
\le
4\cdot \frac d2\cdot \frac3{d^2}
=
\frac6d.
$
Taking square roots gives
$
\left(
\E_{D_1,D_2}\big[
C_{i,j}(R_1\tilde{x},D_2)^2
\big]
\right)^{1/2}
\le
\sqrt{\frac6d},
$
and concludes the proof. \qedhere
}

\end{proof}

\subsection{Proof of Theorem \ref{thm:3rht_vq_error}}\label{sec:lem:3rht_vq_error}

\subsubsection{Main proof.}

\begin{proof}
For clarity, we divide the proof into five steps.

\textbf{Step 1: Deconstructing the matrix multiplication.}
Conditioned on the intermediate vector $y = R_2 \tilde{x}$, the last RHT applies the transformation $U = \sqrt{d} \NH D_3 y = H D_3 y$. It is sufficient to analyze only the first $k$ coordinates of $U$, which we denote as the block $U_{0:k-1}$. 

Let $H_{0:k-1}$ be the sub-matrix consisting of only the first $k$ rows of the Hadamard matrix. We can write our output block as:
\[
U_{0:k-1} = H_{0:k-1} (D_3 y)\,\,.
\]
The vector $D_3 y$ is simply the column vector $y$ with its coordinates multiplied by the random Rademacher signs $\varepsilon_j$ from $D_3$. Therefore, the $j$-th element of $D_3 y$ is exactly $\varepsilon_j y_j$.

By definition, multiplying $H_{0:k-1}$ by a column vector produces a linear combination of the \emph{columns} of $H_{0:k-1}$. Let $H^{(j)}$ denote the $j$-th column of $H_{0:k-1}$ (which is a deterministic vector of length $k$). The multiplication expands to:
\[
U_{0:k-1} = \sum_{j=0}^{d-1} (\varepsilon_j y_j) H^{(j)}\,\,.
\]
To separate the random signs $\varepsilon$ from the deterministic values (i.e., conditioned on $y$), we group the deterministic parts together to define a set of $d$ fixed ``step vectors'', denoted as $V_j$:
\[
V_j = y_j H^{(j)}\,\,.
\]
Notice that the $i$-th coordinate of this step vector is exactly $(V_j)_i = H_{i,j} y_j$. Substituting this definition back into our sum, we have:
\[
U_{0:k-1} = \sum_{j=0}^{d-1} \varepsilon_j V_j\,\,.
\]

Now, calculating the conditional covariance matrix $\Sigma_y$ of this block becomes straightforward. By definition, $\Sigma_y = \E[U_{0:k-1} U_{0:k-1}^\top \mid y]$. Because the signs $\varepsilon$ are independent and zero-mean, all cross-terms $\E[\varepsilon_j \varepsilon_m]$ vanish when $j \neq m$. We are left only with the sum of the outer products of the individual step vectors:
\[
\Sigma_y = \sum_{j=0}^{d-1} \E[\varepsilon_j^2] V_j V_j^\top = \sum_{j=0}^{d-1} V_j V_j^\top\,\,.
\]
Looking at the specific entries of this $k \times k$ covariance matrix:
\begin{itemize}
    \item \textbf{Diagonal entries (Variance):} For coordinate $i$, this is $\sum_j (V_j)_i^2 = \sum_j (H_{i,j} y_j)^2$. Since the Hadamard matrix consists of $\pm 1$, $H_{i,j}^2 = 1$, simplifying the sum to $\sum_j y_j^2 = \|y\|_2^2 = 1$.
    \item \textbf{Off-diagonal entries (Covariance):} For distinct coordinates $i$ and $l$, this is $\sum_j (V_j)_i (V_j)_l = \sum_j H_{i,j} H_{l,j} y_j^2$. Using the Hadamard XOR property, this becomes exactly $\sum_j H_{i \oplus l, j} y_j^2$, which matches exactly the cross-correlation equation $C_{i,l}(R_1 \tilde{x},D_2)$ established in Theorem \ref{thm:3rht_decorr}.
\end{itemize}

\vspace{0.5em}
\textbf{Step 2: Bounding the covariance difference.}
Write the conditioned vector as
\[
y=R_2\tilde{x}=\NH D_2 y^{(1)},\qquad y^{(1)}=R_1\tilde{x}.
\]
The diagonal entries of $\Sigma_y$ are equal to $1$. For $i\neq \ell$, the off-diagonal entry satisfies
\[
(\Sigma_y)_{i\ell}
=
\sum_{j=0}^{d-1}H_{ij}H_{\ell j}y_j^2
=
C_{i,\ell}(y^{(1)},D_2),
\]
where $C_{i,\ell}$ is the conditional covariance from Theorem~\ref{thm:3rht_decorr}. Therefore, by Jensen's inequality and Theorem~\ref{thm:3rht_decorr},

{\color{black}
\begin{align*}
\E\big[\|\Sigma_y-I_k\|_F\big]
&\le
\left(\E\big[\|\Sigma_y-I_k\|_F^2\big]\right)^{1/2} \\
&=
\left(\sum_{i\neq \ell}\E\big[C_{i,\ell}(R_1\tilde{x},D_2)^2\big]\right)^{1/2}
\le
\sqrt{\frac{6k(k-1)}{d}}.
\end{align*}
}

\vspace{0.5em}
\textbf{Step 3: Formalizing the Wasserstein bound via Stein's method.}
To bound the distance between our block $U_{0:k-1}$ and a standard Gaussian $Z \sim \mathcal{N}(0, I_k)$, we use a multivariate normal approximation lemma proved in Appendix~\ref{lemma:mv_stein} that is based on \cite[Theorem 2.15 and Eq.~(3.5)]{raivc2018multivariate}. For a sum of independent, zero-mean random vectors $W = \sum_{j=0}^{d-1} X_j$ with covariance $\Sigma = \sum_{j=0}^{d-1} \E[X_j X_j^\top]$, its 1-Wasserstein distance to a standard Gaussian is bounded by:
\[
W_1\Big(W, \mathcal{N}(0, I_k)\Big) \le A_k \left( \|\Sigma - I_k\|_F + \sum_{j=0}^{d-1} \E\big[\|X_j\|_2^3\big] \right)
\]
where $A_k$ is a constant that depends on $k$.

In our setting, conditioned on the intermediate vector $y$, our block is the sum $U_{0:k-1} = \sum_{j=0}^{d-1} \varepsilon_j V_j$. The summands $X_j = \varepsilon_j V_j$ are independent because the Rademacher signs $\varepsilon_j$ are independent. They are zero-mean ($\E[\varepsilon_j V_j \mid y] = 0$), and their conditional covariance is $\Sigma_y$ as derived in Step 1.

Because $\varepsilon_j \in \{-1, 1\}$, the third absolute moment of each summand simplifies to:
\[
\E\big[\|\varepsilon_j V_j\|_2^3 \mid y\big] = \|V_j\|_2^3 \E\big[|\varepsilon_j|^3\big] = \|V_j\|_2^3 \,\,.
\]
We calculate the length of each step vector as $\|V_j\|_2 = \sqrt{\sum_{i=0}^{k-1} (H_{i,j} y_j)^2} = \sqrt{k} |y_j|$. Summing these cubes over all $d$ steps yields:
\[
\sum_{j=0}^{d-1} \|V_j\|_2^3 = \sum_{j=0}^{d-1} k^{3/2} |y_j|^3 = k^{3/2} \rho_3(y) \,\,,
\]
where $\rho_3(y) = \sum_{j=0}^{d-1} |y_j|^3$ is the third absolute moment of the intermediate vector $y$. 

Substituting the conditional covariance error and the third moments into the above bound gives the conditional Wasserstein distance:
\[
W_1\Big(U_{0:k-1} \mid y, \ \mathcal{N}(0, I_k)\Big) \le A_k \left( \|\Sigma_y - I_k\|_F + k^{3/2} \rho_3(y) \right) \,\,.
\]

To obtain the unconditional bound, we take the expectation over the random vector $y$. By the joint convexity of the Wasserstein metric, the expected distance is bounded by the expectation of the conditional Wasserstein distances, $W_1(U_{0:k-1}, \mathcal{N}(0, I_k)) \le \E_y\big[W_1(U_{0:k-1} \mid y, \mathcal{N}(0, I_k))\big]$. Applying this yields:
\[
W_1\Big(U_{0:k-1}, \ \mathcal{N}(0, I_k)\Big) \le A_k \left( \E\big[\|\Sigma_y - I_k\|_F\big] + k^{3/2} \E\big[\rho_3(y)\big] \right) \,\,.
\]

{\color{black}
From Step 2, we have bounded the expected covariance difference as
\[
\E[\|\Sigma_y-I_k\|_F]
\le
\sqrt{\frac{6k(k-1)}{d}}.
\]
Furthermore, by the 1-RHT smoothing property from Lemma~\ref{lem:l3}, the expected third moment of $y$ is bounded by
\[
\E[\rho_3(y)]
\le
\frac{C_3}{\sqrt d}.
\]
Combining these terms gives our final multi-dimensional bound:
\[
W_1\Big(U_{0:k-1}, \ \mathcal{N}(0, I_k)\Big)
\le
A_k
\left(
\sqrt{\frac{6k(k-1)}{d}}
+
\frac{k^{3/2}C_3}{\sqrt d}
\right)
=
\mathcal{O}_{k}\left(\frac1{\sqrt d}\right).
\]
}

\vspace{0.5em}
\textbf{Step 4: Separating the quadratic part of the VQ error.}
Define the VQ error function
$
L_{\mathcal C}(v)
=
\min_{c\in\mathcal C}\|v-c\|_2^2.
$
We decompose this error into a quadratic term and a codebook-dependent term:
$
L_{\mathcal C}(v)
=
\|v\|_2^2
+
g_{\mathcal C}(v),
$
where
$
g_{\mathcal C}(v)
=
\min_{c\in\mathcal C}
\left(
\|c\|_2^2-2\langle v,c\rangle
\right).
$
Although \(L_{\mathcal C}\) is not globally Lipschitz because of the term \(\|v\|_2^2\), the function \(g_{\mathcal C}\) is globally Lipschitz since for any \(v,w\in\mathbb R^k\),
\[
\begin{aligned}
|g_{\mathcal C}(v)-g_{\mathcal C}(w)|
&\le
\sup_{c\in\mathcal C}
\left|
\left(\|c\|_2^2-2\langle v,c\rangle\right)
-
\left(\|c\|_2^2-2\langle w,c\rangle\right)
\right| \\
&=
2\sup_{c\in\mathcal C}
|\langle v-w,c\rangle| 
\le
2B\|v-w\|_2.
\end{aligned}
\]
Thus, \(g_{\mathcal C}\) is \(2B\)-Lipschitz.

\vspace{0.5em}
\textbf{Step 5: Final gap assembly via quadratic cancellation.}
Let \(Z\sim\mathcal N(0,I_k)\). We first show that the quadratic terms agree in expectation. Conditioned on \(y=R_2\tilde{x}\), each coordinate of the block \(U_{0:k-1}=HD_3y\) has conditional second moment
$
\E\left[U_i^2\mid y\right]
=
\sum_{j=0}^{d-1}H_{ij}^2y_j^2
=
\|y\|_2^2
=
1.
$
Therefore,
$
\E\|U_{0:k-1}\|_2^2
=
\sum_{i=0}^{k-1}\E[U_i^2]
=
k.
$
On the other hand, since \(Z\sim\mathcal N(0,I_k)\),
$
\E\|Z\|_2^2=k.
$
Hence the quadratic parts cancel:
$
\E\|U_{0:k-1}\|_2^2-\E\|Z\|_2^2=0.
$

Using the decomposition from Step 4, we now obtain
\[
\begin{aligned}
\Big|
\E\Big[\min_{c\in\mathcal C}\|U_{0:k-1}-c\|_2^2\Big]
-
\E\Big[\min_{c\in\mathcal C}\|Z-c\|_2^2\Big]
\Big| 
 =
\left|
\E[g_{\mathcal C}(U_{0:k-1})]
-
\E[g_{\mathcal C}(Z)]
\right|.
\end{aligned}
\]
Since \(g_{\mathcal C}\) is \(2B\)-Lipschitz, the definition of the multivariate \(W_1\) distance yields
\[
\left|
\E[g_{\mathcal C}(U_{0:k-1})]
-
\E[g_{\mathcal C}(Z)]
\right|
\le
2B\,
W_1\Big(U_{0:k-1},\mathcal N(0,I_k)\Big).
\]
Applying the Wasserstein estimate from Step 3 yields

{\color{black}
\[
\Big|
\E\Big[\min_{c\in\mathcal C}\|U_{0:k-1}-c\|_2^2\Big]
-
\E\Big[\min_{c\in\mathcal C}\|Z-c\|_2^2\Big]
\Big|
\le
\mathcal{O}_{k,B}\left(\frac1{\sqrt d}\right).
\]
}
This concludes the proof. \qedhere
\end{proof}

\subsubsection{Auxiliary lemma for multivariate normal approximation via Stein's method}
\label{lemma:mv_stein}

We use the following estimate in the proof of
Theorem~\ref{thm:3rht_vq_error}. The point of the lemma is to allow the
covariance of the summands to be close to \(I_k\), rather than exactly equal to
\(I_k\).

\begin{lemma}[Multivariate \(W_1\) normal approximation with covariance mismatch]
\label{lem:mv_stein_bound}
Fix \(k\ge 1\), and let
\(X_1,\ldots,X_n\) be independent mean-zero random vectors in \(\mathbb R^k\)
with finite third moments. Define $S=\sum_{j=1}^n X_j$ and $\Sigma=\operatorname{Cov}(S)$, and let $Z\sim\mathcal N(0,I_k)$. Then
\[
    W_1(S,Z)
    \le
    A_k
    \left(
        \|\Sigma-I_k\|_F
        +
        \sum_{j=1}^n \E\|X_j\|_2^3
    \right),
\]
where \(A_k<\infty\) depends only on \(k\). Specifically, one may take $A_k=4C_k+4\sqrt{k}+1$ where $C_k=11.1+0.83\log k$.

\end{lemma}

\begin{proof}
For clarity, we divide the proof into six steps.

\vspace{0.5em}
\textbf{Step 1: The identity-covariance estimate.}
We first present the standard covariance case that follows from~\cite[Theorem 2.15 and Eq.~(3.5)]{raivc2018multivariate}.

Let \(Y_1,\ldots,Y_n\) be independent mean-zero random vectors in
\(\mathbb R^k\), and suppose
$
    \operatorname{Cov}\left(\sum_{j=1}^n Y_j\right)=I_k.
$
Let
$
    T=\sum_{j=1}^n Y_j.
$
For a test function \(f:\mathbb R^k\to\mathbb R\), write \(M_1(f)\) for its
Lipschitz constant with respect to the Euclidean norm. Then, the bound by \cite[Theorem 2.15 and Eq.~(3.5)]{raivc2018multivariate} yields:
\[
\left|
\E f(T)-\E f(Z)
\right|
\le
M_1(f)
\sum_{j=1}^n
\E\left[
    \|Y_j\|_2^2
    \min\left\{
        4.5,
        C_k\|Y_j\|_2
    \right\}
\right],
\]
where \(C_k=11.1+0.83\log k\). Since
$
    \min\left\{
        4.5,
        C_k\|Y_j\|_2
    \right\}
    \le
    C_k\|Y_j\|_2,
$
we obtain
\[
\left|
\E f(T)-\E f(Z)
\right|
\le
C_k M_1(f)
\sum_{j=1}^n
\E\|Y_j\|_2^3 .
\]
Taking the supremum over all \(1\)-Lipschitz functions \(f\) yields:
$
    W_1(T,Z)
    \le
    C_k
    \sum_{j=1}^n
    \E\|Y_j\|_2^3 .
$

\vspace{0.5em}
\textbf{Step 2: Notation for the general covariance case.}
We now return to the original summands \(X_1,\ldots,X_n\). Define $r=\|\Sigma-I_k\|_F$ and $B=\sum_{j=1}^n \E\|X_j\|_2^3$.
We prove the desired estimate by splitting into two cases according to the size
of \(r\).

\vspace{0.5em}
\textbf{Step 3: The small covariance-mismatch case.}
Assume that
$
    r\le \frac12.
$
Since $\Sigma$ is a covariance matrix, it is symmetric positive semidefinite.
Let
$
    \lambda_1,\ldots,\lambda_k
$
be the eigenvalues of $\Sigma$. Then, the eigenvalues of $\Sigma-I_k$ are
$
    \lambda_1-1,\ldots,\lambda_k-1.
$
Therefore,
$
    r^2
    =
    \|\Sigma-I_k\|_F^2
    =
    \sum_{\ell=1}^k(\lambda_\ell-1)^2.
$
Hence, for every $\ell=1,\ldots,k$,
$$
    |\lambda_\ell-1|
    \le
    \left(\sum_{m=1}^k(\lambda_m-1)^2\right)^{1/2}
    =
    r
    \le
    \frac12.
$$
Consequently,
$
    \frac12\le \lambda_\ell\le \frac32
    \qquad
    \text{for all } \ell=1,\ldots,k.
$
In particular, $\Sigma$ is positive definite, so $\Sigma^{-1/2}$ is well-defined.

Define the standardized summands $\widetilde X_j=\Sigma^{-1/2}X_j$ and $\widetilde S=\sum_{j=1}^n \widetilde X_j = \Sigma^{-1/2}S$.
The random vectors $\widetilde X_1,\ldots,\widetilde X_n$ are independent and
mean zero. Moreover,
$$
    \operatorname{Cov}(\widetilde S)
    =
    \Sigma^{-1/2}\operatorname{Cov}(S)\Sigma^{-1/2}
    =
    \Sigma^{-1/2}\Sigma\Sigma^{-1/2}
    =
    I_k.
$$
Applying the identity-covariance estimate from Step 1 gives
$$
    W_1(\widetilde S,Z)
    \le
    C_k
    \sum_{j=1}^n
    \mathbb E\|\Sigma^{-1/2}X_j\|_2^3.
$$

We now bound the third moments after this linear change of variables. For a
matrix $A$, define its Euclidean operator norm by
$
    \|A\|_{\mathrm{op}}
    =
    \sup_{\|u\|_2=1}\|Au\|_2.
$
Equivalently, $\|A\|_{\mathrm{op}}$ is the smallest number $L$ such that
$
    \|Av\|_2\le L\|v\|_2
    \,
    \text{for every } v\in\mathbb R^k.
$
Indeed, if $v\neq 0$, then $u=v/\|v\|_2$ has unit norm, and therefore
$
    \|Av\|_2
    =
    \|v\|_2
    \left\|A\frac{v}{\|v\|_2}\right\|_2
    \le
    \|v\|_2\|A\|_{\mathrm{op}}.
$
The case $v=0$ is trivial.

Applying this with $A=\Sigma^{-1/2}$ and $v=X_j$, we obtain
$
    \|\Sigma^{-1/2}X_j\|_2
    \le
    \|\Sigma^{-1/2}\|_{\mathrm{op}}\|X_j\|_2.
$
Cubing both sides and taking expectations gives
$
    \mathbb E\|\Sigma^{-1/2}X_j\|_2^3
    \le
    \|\Sigma^{-1/2}\|_{\mathrm{op}}^3
    \mathbb E\|X_j\|_2^3.
$
Summing over $j$ yields
$
    \sum_{j=1}^n
    \mathbb E\|\Sigma^{-1/2}X_j\|_2^3
    \le
    \|\Sigma^{-1/2}\|_{\mathrm{op}}^3 B.
$
Thus,
$$
    W_1(\widetilde S,Z)
    \le
    C_k
    \|\Sigma^{-1/2}\|_{\mathrm{op}}^3 B.
$$

Since the eigenvalues of $\Sigma$ lie in $[1/2,3/2]$, the eigenvalues of
$\Sigma^{-1/2}$ lie in
$
    \left[\sqrt{\frac{2}{3}},\sqrt{2}\right].
$
Because $\Sigma^{-1/2}$ is symmetric positive definite, its operator norm is
its largest eigenvalue. Therefore,
$
    \|\Sigma^{-1/2}\|_{\mathrm{op}}
    \le
    \sqrt{2}.
$
Substituting this into the preceding bound gives
$$
    W_1(\widetilde S,Z)
    \le
    2\sqrt{2}\,C_k B.
$$

\vspace{0.5em}
\textbf{Step 4: Returning from the standardized covariance to \(I_k\).}
By the triangle inequality,
$$
    W_1(S,Z)
    \le
    W_1(S,\Sigma^{1/2}Z)
    +
    W_1(\Sigma^{1/2}Z,Z).
$$
For the first term, use the identity
$
    S=\Sigma^{1/2}\widetilde S.
$
We also use the following Lipschitz property of $W_1$: for any
matrix $A$ and any random vectors $P,Q$ in $\mathbb R^k$,
$$
    W_1(AP,AQ)
    \le
    \|A\|_{\mathrm{op}} W_1(P,Q).
$$
To see this directly from the dual definition of $W_1$, let $f$ be any
$1$-Lipschitz function. Define
$
    g(v):=f(Av).
$
Then $g$ is $\|A\|_{\mathrm{op}}$-Lipschitz, since
$$
    |g(v)-g(w)|
    =
    |f(Av)-f(Aw)|
    \le
    \|A(v-w)\|_2
    \le
    \|A\|_{\mathrm{op}}\|v-w\|_2.
$$
If $\|A\|_{\mathrm{op}}=0$, then the desired
inequality is trivial. Otherwise, set
$
    h(v):=\frac{g(v)}{\|A\|_{\mathrm{op}}}.
$
Then $h$ is $1$-Lipschitz. Hence, by the dual definition of $W_1$,
$$
\begin{aligned}
    \left|
        \mathbb E f(AP)-\mathbb E f(AQ)
    \right|
    &=
    \left|
        \mathbb E g(P)-\mathbb E g(Q)
    \right|  \\
    &=
    \|A\|_{\mathrm{op}}
    \left|
        \mathbb E h(P)-\mathbb E h(Q)
    \right|  
    \le
    \|A\|_{\mathrm{op}} W_1(P,Q).
\end{aligned}
$$
Taking the supremum over all $1$-Lipschitz functions $f$ gives
$$
    W_1(AP,AQ)
    \le
    \|A\|_{\mathrm{op}} W_1(P,Q).
$$

Applying this property with $A=\Sigma^{1/2}$, $P=\widetilde S$, and $Q=Z$ gives
$$
    W_1(S,\Sigma^{1/2}Z)
    =
    W_1(\Sigma^{1/2}\widetilde S,\Sigma^{1/2}Z)
    \le
    \|\Sigma^{1/2}\|_{\mathrm{op}}
    W_1(\widetilde S,Z).
$$
Since the eigenvalues of $\Sigma$ lie in $[1/2,3/2]$, the eigenvalues of
$\Sigma^{1/2}$ lie in
$
    \left[\frac{1}{\sqrt{2}},\sqrt{\frac32}\right].
$
Thus,
$
    \|\Sigma^{1/2}\|_{\mathrm{op}}
    \le
    \sqrt{\frac32}.
$
Combining this with the bound from Step 3 gives
$$
    W_1(S,\Sigma^{1/2}Z)
    \le
    \sqrt{\frac32}\cdot 2\sqrt{2}\,C_kB
    =
    2\sqrt{3}\,C_kB
    \le
    4C_kB.
$$

For the second term, we compare $\Sigma^{1/2}Z$ and $Z$ using the same Gaussian
vector $Z$. For any $1$-Lipschitz function $f$,
$
    \left|
        \mathbb E f(\Sigma^{1/2}Z)-\mathbb E f(Z)
    \right|
    \le
    \mathbb E
    \left[
        \left|f(\Sigma^{1/2}Z)-f(Z)\right|
    \right]
    \le
    \mathbb E\|(\Sigma^{1/2}-I_k)Z\|_2.
$
Taking the supremum over all $1$-Lipschitz $f$ yields,
$$
    W_1(\Sigma^{1/2}Z,Z)
    \le
    \mathbb E\|(\Sigma^{1/2}-I_k)Z\|_2.
$$
By Cauchy--Schwarz,
$$
    \mathbb E\|(\Sigma^{1/2}-I_k)Z\|_2
    \le
    \left(
        \mathbb E\|(\Sigma^{1/2}-I_k)Z\|_2^2
    \right)^{1/2}.
$$
Since $Z\sim\mathcal N(0,I_k)$,
$$
    \mathbb E\|(\Sigma^{1/2}-I_k)Z\|_2^2
    =
    \|\Sigma^{1/2}-I_k\|_F^2.
$$
Indeed, if $M=\Sigma^{1/2}-I_k$, then
$$
    \mathbb E\|MZ\|_2^2
    =
    \mathbb E[Z^\top M^\top MZ]
    =
    \operatorname{tr}(M^\top M\mathbb E[ZZ^\top])
    =
    \operatorname{tr}(M^\top M)
    =
    \|M\|_F^2.
$$
Therefore,
$$
    W_1(\Sigma^{1/2}Z,Z)
    \le
    \|\Sigma^{1/2}-I_k\|_F.
$$

It remains to compare $\|\Sigma^{1/2}-I_k\|_F$ with
$\|\Sigma-I_k\|_F$. Since $\Sigma$ is symmetric positive semidefinite, it has an
orthonormal eigenbasis. In that basis, the eigenvalues of
$\Sigma^{1/2}-I_k$ are $\sqrt{\lambda_\ell}-1$, while the eigenvalues of
$\Sigma-I_k$ are $\lambda_\ell-1$. For every $\lambda_\ell\ge 0$,
$$
    |\sqrt{\lambda_\ell}-1|
    =
    \frac{|\lambda_\ell-1|}{\sqrt{\lambda_\ell}+1}
    \le
    |\lambda_\ell-1|.
$$
Hence,
$
    \|\Sigma^{1/2}-I_k\|_F
    \le
    \|\Sigma-I_k\|_F
    =
    r.
$
Thus,
$
    W_1(\Sigma^{1/2}Z,Z)
    \le
    r.
$

Combining the two terms, in the case $r\le 1/2$ we obtain
$$
    W_1(S,Z)
    \le
    r+4C_kB.
$$

\vspace{0.5em}
\textbf{Step 5: The large covariance-mismatch case.}
Assume that
$
    r>\frac12.
$
We use a crude bound that depends only on the second moments. Let $f$ be any
$1$-Lipschitz function. Subtracting the constant $f(0)$ does not change the
difference of expectations, so
$$
\begin{aligned}
    \left|\mathbb E f(S)-\mathbb E f(Z)\right|
    &=
    \left|
        \mathbb E[f(S)-f(0)]
        -
        \mathbb E[f(Z)-f(0)]
    \right| \\
    &\le
    \mathbb E|f(S)-f(0)|
    +
    \mathbb E|f(Z)-f(0)| \le
    \mathbb E\|S\|_2+\mathbb E\|Z\|_2.
\end{aligned}
$$
Taking the supremum over all $1$-Lipschitz $f$ gives
$$
    W_1(S,Z)
    \le
    \mathbb E\|S\|_2+\mathbb E\|Z\|_2.
$$
By Jensen's inequality, $\mathbb E\|S\|_2 \le \left(\mathbb E\|S\|_2^2\right)^{1/2}$ and $\mathbb E\|Z\|_2 \le \left(\mathbb E\|Z\|_2^2\right)^{1/2}$.

Since $\mathbb E S=0$ and $\operatorname{Cov}(S)=\Sigma$,
$
    \mathbb E\|S\|_2^2
    =
    \operatorname{tr}(\Sigma).
$
Also, since $Z\sim\mathcal N(0,I_k)$,
$
    \mathbb E\|Z\|_2^2
    =
    k.
$
Therefore,
$$
    W_1(S,Z)
    \le
    \sqrt{\operatorname{tr}(\Sigma)}+\sqrt{k}.
$$

We now bound $\operatorname{tr}(\Sigma)$. Since
$
    \operatorname{tr}(\Sigma)
    =
    k+\operatorname{tr}(\Sigma-I_k),
$
and since
$
    \operatorname{tr}(\Sigma-I_k)
    =
    \langle \Sigma-I_k,I_k\rangle_F,
$
Cauchy's inequality for the Frobenius inner product yields,
$
    \operatorname{tr}(\Sigma-I_k)
    \le
    \|\Sigma-I_k\|_F\|I_k\|_F
    =
    r\sqrt{k}.
$
Thus,
$
    \operatorname{tr}(\Sigma)
    \le
    k+r\sqrt{k}.
$
Consequently,
$
    \sqrt{\operatorname{tr}(\Sigma)}
    \le
    \sqrt{k+r\sqrt{k}}
    \le
    \sqrt{k}+r,
$
because
$
    (\sqrt{k}+r)^2
    =
    k+2r\sqrt{k}+r^2
    \ge
    k+r\sqrt{k}.
$
Hence,
$
    W_1(S,Z)
    \le
    2\sqrt{k}+r.
$
Since $r>1/2$,
$
    2\sqrt{k}
    \le
    4\sqrt{k}\,r.
$

Therefore, in the case $r>1/2$,
$$
    W_1(S,Z)
    \le
    (4\sqrt{k}+1)r.
$$

\vspace{0.5em}
\textbf{Step 6: Combining the two cases.}
If $r\le 1/2$, Step 4 yields,
$$
    W_1(S,Z)
    \le
    r+4C_kB
    \le
    (4C_k+4\sqrt{k}+1)(r+B).
$$
If $r>1/2$, Step 5 yields,
$$
    W_1(S,Z)
    \le
    (4\sqrt{k}+1)r
    \le
    (4C_k+4\sqrt{k}+1)(r+B).
$$
Thus, in all cases,
$$
    W_1(S,Z)
    \le
    A_k(r+B),
$$
where
$
    A_k=4C_k+4\sqrt{k}+1.
$
Substituting the definitions of $r$ and $B$ gives
$$
    W_1(S,Z)
    \le
    A_k
    \left(
        \|\Sigma-I_k\|_F
        +
        \sum_{j=1}^n \mathbb E\|X_j\|_2^3
    \right).
$$
This proves the lemma.
\end{proof}

{
\section{Additional experimental results}\label{app:experiments}
\subsection{Fixed-codebook HIGGS experiment}\label{app:Fixed-codebook HIGGS experiment}

We compare one, two, and three RHTs using the released 3-bit HIGGS $p=2$ codebook. For each $d\in\{128,512,2048,8192\}$, we place 32 consecutive independent Rademacher entries at a random position in an otherwise zero vector. We quantize pairs of rotated coordinates with their nearest codeword. We repeat this experiment over 12 random inputs, each rotated with 8 different seeds.

Table~\ref{tab:higgs-depth} shows that three RHTs reduce vNMSE by $6.74$--$8.59\%$ compared with one RHT and by $1.77$--$3.10\%$ compared with two. At $d=128,512,2048$, the measured relative vNMSE difference from URR is at most $0.06\%$ in magnitude.

\begin{table}[htbp]

\centering\small\setlength{\tabcolsep}{4pt}
\begin{tabular}{rlll}
\toprule
$d$ & 3 vs.\ 1 RHT & 3 vs.\ 2 RHTs & 3 RHTs vs.\ URR \\
\midrule
128 & $8.59\;[8.23,8.95]$ & $1.77\;[1.20,2.34]$ & $-0.06\;[-0.23,+0.10]$ \\
512 & $7.68\;[7.21,8.14]$ & $2.67\;[2.25,3.09]$ & $+0.04\;[-0.32,+0.39]$ \\
2048 & $7.15\;[5.89,8.39]$ & $2.69\;[2.21,3.16]$ & $+0.05\;[-0.30,+0.40]$ \\
8192 & $6.74\;[4.60,8.84]$ & $3.10\;[2.61,3.59]$ & Timed out \\
\bottomrule
\end{tabular}
\captionsetup{position=bottom,skip=7pt}
\caption{
{Relative HIGGS vNMSE reductions from three RHTs over one and two RHTs, and differences from URR. Values are percentages with paired 95\% confidence intervals.}}
\label{tab:higgs-depth}
\end{table}

\subsection{Adaptive scalar and vector checks}\label{app:Adaptive scalar and vector checks}

{At $d=512$, we use the scalar check in Section~\ref{sec:adaptive} to choose one or two RHTs for EDEN, and the vector check to choose two or three for HIGGS. We set $c_3=3.9$ and $c_\infty=2.5$ using separate calibration inputs, then keep them fixed. We repeat each experiment over 12 random input batches of 64 vectors, each rotated with 8 different seeds.}

\paragraph{Scalar check.}

{For 3-bit EDEN, each random input batch contains 58 Gaussian vectors and six sparse vectors. For each sparse vector, we place 32 consecutive independent Rademacher entries at a random position in an otherwise zero vector. The check uses one RHT for every Gaussian vector and two for every sparse vector. In Table~\ref{tab:adaptive-scalar}, adaptive vNMSE is $0.03437$, close to $0.03434$ from always using two RHTs, while only $9.375\%$ of vectors receive a second RHT. On entirely sparse inputs, two RHTs reduce vNMSE by $14.99\%$ compared with one.}
\begin{table}[htbp]

\centering\small
\begin{tabular}{lrrr}
\toprule
Method & vNMSE & Change vs. one RHT (95\% CI) & Second-RHT fraction \\
\midrule
Always one & 0.03505 & --- & $0\%$ \\
Adaptive one/two & 0.03437 & $-1.95\%\;[-2.32,-1.58]$ & $9.375\%$ \\
Always two & 0.03434 & $-2.04\%\;[-2.55,-1.52]$ & $100\%$ \\
\bottomrule
\end{tabular}
\captionsetup{position=bottom,skip=7pt}
\caption{
{3-bit EDEN on mixed Gaussian and sparse inputs. Changes and paired 95\% confidence intervals are in percent.}}
\label{tab:adaptive-scalar}
\end{table}

\paragraph{Vector check.}

{For 3-bit HIGGS $p=2$, each random input batch contains 58 Gaussian vectors and six sparse vectors. For each sparse vector, we place 32 consecutive independent Rademacher entries at a random position in an otherwise zero vector. For 2-bit HIGGS, each random input batch contains 58 dense vectors with independent Rademacher entries and six sparse vectors, each formed by placing two consecutive independent Rademacher entries at a random position in an otherwise zero vector. We quantize pairs of rotated coordinates with their nearest codeword. In Table~\ref{tab:adaptive-vector}, adaptive vNMSE is close to always using three RHTs, while only $12.1\%$ and $9.4\%$ of vectors receive a third RHT. On entirely sparse inputs, three RHTs improve over both one and two RHTs.}
\begin{table}[htbp]

\centering\small
\begin{tabular}{lrr}
\toprule
 & 3-bit, 32 nonzeros & 2-bit, 2 nonzeros \\
\midrule
Three vs. one, all sparse & $-7.47\%\;[-8.03,-6.92]$ & $-23.14\%\;[-23.28,-23.01]$ \\
Three vs. two, all sparse & $-2.74\%\;[-3.16,-2.30]$ & $-13.39\%\;[-14.25,-12.52]$ \\
Third-RHT fraction, mixed & $12.1\%$ & $9.4\%$ \\
Adaptive vNMSE, mixed & 0.02949 & 0.10719 \\
Always-three vNMSE, mixed & 0.02949 & 0.10728 \\
\bottomrule
\end{tabular}
\medskip
\caption{
{Adaptive HIGGS results on sparse and mixed inputs. Brackets give paired 95\% confidence intervals in percent.}}
\label{tab:adaptive-vector}
\end{table}

\paragraph{Real tensors.}

{We test 1968 vectors of dimension 512 from 21 Qwen2.5-0.5B tensors: weights, hidden activations, KV cache, and gradients. We apply 3-bit EDEN and 3-bit HIGGS $p=2$ using the thresholds above. We repeat each experiment with 8 different seeds for each input vector.}
\begin{table}[htbp]

\centering\small
\begin{tabular}{lrrr}
\toprule
Data type & Vectors & Scalar: one / two & Vector: two / three \\
\midrule
Weights & 768 & 764 / 4 & 520 / 248 \\
Activations & 336 & 171 / 165 & 1 / 335 \\
KV cache & 96 & 96 / 0 & 14 / 82 \\
Gradients & 768 & 577 / 191 & 62 / 706 \\
\midrule
Total & 1968 & 1608 / 360 & 597 / 1371 \\
\bottomrule
\end{tabular}
\medskip
\caption{
{RHT counts chosen for the Qwen2.5-0.5B vectors.}}
\label{tab:real-tensor-routing}
\end{table}

{The scalar check skips the second RHT for $81.7\%$ of vectors, and the vector check skips the third for $30.3\%$ (Table~\ref{tab:real-tensor-routing}). Adaptive HIGGS and always using three RHTs both keep vNMSE below $0.05$ in every trial. The check selects a third RHT for the only two-RHT result above $0.05$, reducing vNMSE from $0.05049$ to $0.02887$.}

\subsection{Checks of scalar marginal approximation}\label{app:scalar-marginal-checks}

{We test $e_0$ and $(e_0+e_1)/\sqrt2$ for $d\in\{64,128,\ldots,4096\}$. We compare $\sqrt d$ times an output coordinate with a standard Gaussian using Kolmogorov and 1-Wasserstein distances. The one- and two-RHT distances are exact; for three RHTs, we repeat the experiment with 8 batches of 128 different seeds and average the resulting estimates.

Tables~\ref{tab:scalar-check-e0} and~\ref{tab:scalar-check-two} multiply the distances by $\sqrt d$. For both inputs, the two-RHT values stay nearly constant, consistent with the $d^{-1/2}$ bounds. Three RHTs give smaller distances at every dimension.}

\begin{table}[htbp]

\centering\small\setlength{\tabcolsep}{4pt}
\begin{tabular}{rrrrrrr}
\toprule
 & \multicolumn{3}{c}{$\sqrt d\,d_K$} & \multicolumn{3}{c}{$\sqrt d\,W_1$} \\
$d$ & One & Two & Three & One & Two & Three \\
\midrule
64 & 2.73076 & 0.39740 & 0.17195 & 4.28302 & 0.50058 & 0.17651 \\
128 & 3.86188 & 0.39820 & 0.13101 & 6.05710 & 0.50030 & 0.13078 \\
256 & 5.46152 & 0.39857 & 0.10259 & 8.56603 & 0.50015 & 0.09999 \\
512 & 7.72376 & 0.39874 & 0.07733 & 12.11420 & 0.50007 & 0.07217 \\
1024 & 10.92304 & 0.39886 & 0.06589 & 17.13206 & 0.50005 & 0.06950 \\
2048 & 15.44751 & 0.39892 & 0.05399 & 24.22840 & 0.50004 & 0.05171 \\
4096 & 21.84608 & 0.39893 & 0.05022 & 34.26413 & 0.50003 & 0.05298 \\
\bottomrule
\end{tabular}
\caption{ {Distances to a standard Gaussian for $e_0$, scaled by $\sqrt d$.}}
\label{tab:scalar-check-e0}
\end{table}

\begin{table}[htbp]

\centering\small\setlength{\tabcolsep}{4pt}
\begin{tabular}{rrrrrrr}
\toprule
 & \multicolumn{3}{c}{$\sqrt d\,d_K$} & \multicolumn{3}{c}{$\sqrt d\,W_1$} \\
$d$ & One & Two & Three & One & Two & Three \\
\midrule
64 & 2.00000 & 0.55979 & 0.23849 & 3.00821 & 0.70873 & 0.23901 \\
128 & 2.82843 & 0.56201 & 0.18322 & 4.25425 & 0.70793 & 0.17458 \\
256 & 4.00000 & 0.56313 & 0.12687 & 6.01642 & 0.70753 & 0.11949 \\
512 & 5.65685 & 0.56367 & 0.09689 & 8.50850 & 0.70731 & 0.09130 \\
1024 & 8.00000 & 0.56390 & 0.08693 & 12.03283 & 0.70720 & 0.07932 \\
2048 & 11.31371 & 0.56407 & 0.06581 & 17.01699 & 0.70718 & 0.06726 \\
4096 & 16.00000 & 0.56416 & 0.05262 & 24.06566 & 0.70716 & 0.05566 \\
\bottomrule
\end{tabular}
\caption{ {Distances to a standard Gaussian for $(e_0+e_1)/\sqrt2$, scaled by $\sqrt d$.}}
\label{tab:scalar-check-two}
\end{table}

}

\end{document}